\documentclass[10pt,journal]{IEEEtran}

\usepackage{amsmath,amssymb,amsfonts}
\usepackage{amsthm}
\usepackage{bm}
\usepackage{mathtools}
\usepackage{booktabs}
\usepackage{graphicx}
\usepackage[caption=false,font=footnotesize]{subfig}
\usepackage{placeins}
\usepackage{dblfloatfix}

\usepackage{algorithm}
\usepackage{algpseudocode}
\usepackage{array}
\usepackage{multirow}
\usepackage{microtype}
\usepackage{cite}
\usepackage{url}


\newcommand{\R}{\mathbb{R}}
\newcommand{\X}{\mathcal{X}}
\newcommand{\Ocal}{\mathcal{O}}
\newcommand{\Mcal}{\mathcal{M}}
\newcommand{\Rcal}{\mathcal{R}}
\newcommand{\Bcal}{\mathcal{B}}
\newcommand{\Ecal}{\mathcal{E}}
\newcommand{\Zcal}{\mathcal{Z}}
\newcommand{\Hcal}{\mathcal{H}}
\newcommand{\Pcal}{\mathcal{P}}

\newcommand{\ThetaK}{\Theta_{\mathrm{KPD}}}
\newcommand{\od}{\circ}
\newcommand{\had}{\odot}
\newcommand{\norm}[1]{\left\lVert #1\right\rVert}

\newcommand{\diag}{\operatorname{Diag}}
\newcommand{\vecop}{\operatorname{vec}}

\newcommand{\argmin}{\operatorname*{arg\,min}}

\newtheorem{proposition}{Proposition}
\newtheorem{remark}{Remark}

\title{Kronecker-Structured Video Completion with Cross-Fitted Self-Supervised Residual Learning}

\author{Qi Qi, Yanfei Wang, Yanjun Cui, and Jun-e Feng%
\thanks{This work was supported by National Natural Science Foundation of China under Grant 62673293, 62350037, and the Research Fund for the Taishan Scholar Project of Shandong Province of China (tstp20221103). (\textit{Corresponding author: Jun-e Feng}.)}%
\thanks{Qi Qi is with the Mathematical Research Center,
Shandong University, Jinan 250100, China
(e-mail: qqiqi@mail.sdu.edu.cn).}%
\thanks{Yanfei Wang is with the School of Mathematics,
Shandong University, Jinan 250100, China
(e-mail: wangyanfei@mail.sdu.edu.cn).}%
\thanks{Yanjun Cui is with Dartmouth College,
Hanover, NH 03755, USA
(e-mail: yanjun.cui.gr@dartmouth.edu).}%
\thanks{Jun-e Feng is with the School of Mathematics,
Shandong University, Jinan 250100, China,
and also with Semi-tensor Product of Matrices:
Theory and Applications, School of Mathematics and
Systems Science, Liaocheng University, Liaocheng
252000, China (e-mail: fengjune@sdu.edu.cn).}%
}

\begin{document}
\maketitle

\begin{abstract}
Tube-wise video completion reconstructs temporal trajectories at spatial locations unobserved in every frame. We propose a two-stage framework combining regularized masked Kronecker product decomposition (KPD) with cross-fitted self-supervised residual learning. First, we estimate a compact structural backbone directly from observed entries using cyclic component-wise least-squares updates, spatial and temporal factor regularization, and regularizer-optimal scale balancing. The exact updates yield a nonincreasing objective. Second, we address the mismatch between in-sample fitting errors and the prediction errors encountered at missing tubes. We partition observed tubes into folds and construct residual targets using backbone models fitted without the corresponding target tubes. A lightweight 3-D convolutional neural network then predicts corrections from the backbone, visible residual context, observation mask, and local detail features. Exact data consistency preserves observed measurements. Experiments include ablations, sensitivity analysis, grayscale and RGB video completion, and structured missing patterns. Our method achieves the highest mean PSNR and SSIM at the higher tested sampling rates on the grayscale benchmark. On eight selected DAVIS RGB sequences, it achieves the highest mean missing-region PSNR and full-frame SSIM and the lowest mean missing-tube temporal-gradient MAE at every evaluated sampling rate.
\end{abstract}


\section{Introduction}
\label{sec:introduction}

Video completion reconstructs missing visual information from partial
measurements. A video contains both spatial structure and temporal evolution,
which provide complementary information for reconstruction. This paper considers tube-wise missingness: a spatial
location is either observed throughout the video or missing across all
frames. Thus, the complete temporal trajectory at each unobserved location
is unavailable. Recovering these trajectories requires information from the observed
spatial tubes and the relationships among video frames. Tensor
representations describe these relationships by retaining the spatial
and temporal modes of the data
\cite{kolda2009tensor}. Tensor completion extends the principle of low-rank
matrix recovery \cite{candes2009matrix,cai2010svt} to this multiway setting,
using compact structure to infer missing observations.

Low-rank tensor completion can be formulated using tensor unfoldings or direct factor models. Tensor-unfolding nuclear norms impose low-rank constraints on matrix representations \cite{liu2013tensor}, while weighted CP factorization fits separable components to observed entries \cite{acar2011scalable}. Factor smoothness \cite{yokota2016smooth}, Tensor-Train models \cite{bengua2017tt}, and joint CP--Tucker constraints \cite{liu2020cptucker} offer different ways to represent multiway structure. Adaptive low-rank estimation and direct factor optimization \cite{zhang2020adaptive,zhou2018factorization} provide additional formulations, while fully connected and wheel-shaped tensor networks \cite{zheng2021fctn,wu2022tensorwheel} represent more complex interactions among modes. These approaches motivate fitting a compact structural model directly to the available measurements.

Other tensor-completion methods exploit low-rank structure in transformed coordinates. The tensor--tensor product \cite{kilmer2011factorization} underlies tensor-SVD completion and its recovery analysis \cite{zhang2014novel,zhang2017exact}. The twist tensor nuclear norm reorganizes spatiotemporal information for video recovery \cite{hu2017twist}. Invertible transforms \cite{kernfeld2015invertible,lu2019invertible}, framelet and patched-tube representations \cite{jiang2020framelet,ng2020patched}, and robust or structured tensor nuclear-norm models \cite{lu2020trpca,song2020robust,xu2022fasttnn} extend this approach. For both factor-based and transform-based methods, recovery depends on how incomplete observations constrain the estimated structure.

Visual reconstruction also benefits from modeling signal regularity and relationships beyond those captured by global rank constraints. Coupled transforms \cite{wang2021coupled}, global--local--nonlocal priors \cite{zhao2022complementary}, temporal smoothness \cite{wang2023tctv}, low-rank dictionary coding \cite{jiang2023dictionary}, and smooth factor representations \cite{wu2024smooth} incorporate complementary signal information. Patch-tracked streaming tensor-ring models use temporal relationships among video patches \cite{he2022patchtracking}. Other structural approaches include Kronecker-basis sparsity \cite{xie2018kbr}, tensor low-rank representations \cite{zhou2021tlrr,yang2024gtlrr}, and higher-order t-SVD \cite{qin2022highorder}. These models seek to preserve detail while compactly representing high-dimensional video data.

Learned tensor representations further adapt structural models to the observed signal. Examples include self-supervised nonlinear transforms \cite{luo2022s2ntnn}, hierarchical low-rank factorization \cite{luo2022hlrtf}, and learnable group-tube, coupled, and adaptive transforms \cite{li2023grouptube,wang2024conot,wu2024salts}. Continuous tensor functions \cite{luo2024lrtfr}, deep rank-one factorization \cite{li2025drotff}, learned low-rank metrics \cite{wang2025otlrm}, multidimensional nonlinear transforms \cite{shi2025multidimensional}, DELTA \cite{yang2026delta}, and compact tensor-wheel modeling \cite{yang2026twlrs} further develop learned tensor representations. For tube-wise video completion, two related tasks remain: estimating a compact structural model from incomplete observations and predicting details that the model does not recover.

We address structural estimation through a regularized masked Kronecker product decomposition (KPD) algorithm. Cheng's component least-squares algorithm (CLSA) constructs finite-sum Kronecker decompositions through nearest-Kronecker-product subproblems \cite{cheng2026clsa}. For third-order videos, the vectorized KPD uses the same separable rank-$K$ parameterization as CP decomposition. Our contribution lies in the masked formulation and its optimization. All components are fitted to a joint observed-entry objective with spatial and temporal regularization and revisited through cyclic least-squares updates. Equivalent factor scalings are balanced to minimize the regularization penalty. This procedure estimates the structural backbone directly from observed data. We also establish that its exact cyclic updates produce nonincreasing objective values.

A compact structural backbone captures the main spatiotemporal structure, while a learned residual can refine details beyond the chosen rank. Neural video inpainting also uses information across frames. For example, graph-memory modeling aggregates cross-frame information for reconstruction \cite{zhou2024graphmemory}. In our tube-wise setting, the predictor combines the masked KPD backbone with residual context from observed spatial tubes. Training the residual predictor requires targets that correspond to prediction errors at unobserved locations.

Self-supervised restoration provides useful precedents for learning from
incomplete measurements. Deep Image Prior fits an image-specific
representation from a single degraded observation
\cite{ulyanov2018dip}, whereas Noise2Self and Self2Self use held-out or
randomized observations to construct training signals
\cite{batson2019noise2self,quan2020self2self}. Auxiliary inpainting has
also been used to create self-supervision in multi-focus image fusion
\cite{lin2025fusion2void}. These approaches motivate constructing supervision by temporarily
hiding a subset of the observed data. In a
structure--residual framework, however, the structural estimator must also
exclude the held-out observations. If a tube contributes to fitting the backbone and is
then reused to supervise the residual predictor, its target is an
\emph{in-sample fitting error}. At a genuinely missing tube, the desired
correction concerns the backbone's \emph{out-of-sample prediction error}.
Masking a tube only at the residual-network input does not separate these
two errors if the backbone was fitted using that tube.

We construct residual supervision by applying sample splitting and
cross-fitting \cite{chernozhukov2018dml} at the level of the structural
estimator. Observed tubes are divided into folds, with each fold
predicted by a backbone fitted without its target tubes. The resulting
residuals measure held-out structural prediction errors and serve as
training targets for a lightweight per-video 3-D convolutional neural
network (CNN). At inference, the backbone is fitted to all available
tubes, and the trained network predicts corrections at missing locations.
This procedure separates global structure estimation from the
construction of residual supervision without requiring complete training
videos.

Specifically, we express the reconstruction as a structural backbone
$\Bcal_K$ and a learned residual correction $\Rcal$:
\begin{equation}
    \X \approx \Bcal_K + \Rcal.
    \label{eq:intro_structure_residual}
\end{equation}
The underlying KPD reconstruction is a finite sum of
separable components,
\begin{equation}
    \widetilde{\X}_K = \sum_{r=1}^{K}\bm a_r\od\bm b_r\od\bm c_r,
    \qquad
    \vecop(\bm a_r\od\bm b_r\od\bm c_r)
    = \bm c_r\otimes\bm b_r\otimes\bm a_r.
    \label{eq:intro_kpd_form}
\end{equation}
For normalized intensities, the structural backbone is
$\Bcal_K=\Pi_{[0,1]}(\widetilde{\X}_K)$, where $\Pi_{[0,1]}$ clips
entries to $[0,1]$. Stage~I estimates the factors directly from observed
entries using regularized masked component least squares. To construct
Stage~II training targets, the observed spatial tubes are partitioned into
folds $\{\Gamma_p^{(f)}\}_{f=1}^{F}$. A fold-specific backbone
$\Bcal_K^{(f)}$ is fitted with the tubes in $\Gamma_p^{(f)}$ withheld,
and the target residual is computed as
\begin{equation}
    \Rcal_p^{(f)}
    = \Pcal_{\Omega_p^{(f)}}
      \bigl(\Ocal-\Bcal_K^{(f)}\bigr),
    \label{eq:intro_crossfit_target}
\end{equation}
where $\Omega_p^{(f)}$ contains all voxels in the held-out tubes.
The network learns these residual targets from four inputs: the bounded
backbone, visible residuals, the sampling mask, and a local spatial
high-frequency feature. The learned correction is then combined with
the backbone, and an exact data-consistency step preserves the observed
video measurements.

The main contributions are as follows:
\begin{itemize}
    \item \textbf{Regularized masked KPD estimation.}
    We formulate a joint observed-entry KPD objective with spatial and
    temporal factor regularization and derive cyclic component-wise
    least-squares updates. Regularizer-optimal scale balancing preserves
    each represented component while minimizing its factor penalty.
    Proposition~\ref{prop:monotonicity} establishes the nonincreasing
    objective property of the exact update steps.

    \item \textbf{Cross-fitted residual supervision.}
    We construct residual targets by partitioning observed tubes into
    folds and excluding each target fold from its corresponding masked
    KPD fit. A lightweight residual predictor is trained on these
    held-out structural prediction errors rather than in-sample fitting
    residuals. Proposition~\ref{prop:no-leakage} establishes the
    hold-out property of the target construction.

    \item \textbf{Integrated video completion framework.}
    We combine masked KPD estimation and cross-fitted residual learning
    with a per-video 3-D CNN and exact data consistency. Experiments on
    grayscale and RGB videos examine representation efficiency,
    residual-supervision and input ablations, sampling-rate comparisons,
    and structured missing patterns.
\end{itemize}

The remainder of this paper is organized as follows.
Section~\ref{sec:problem} introduces the observation model and
Kronecker-structured representation. Section~\ref{sec:masked} develops
the regularized masked KPD estimator and its update properties.
Section~\ref{sec:crossfit} presents cross-fitted residual supervision,
network training, and data-consistent inference.
Section~\ref{sec:experiments} reports ablation studies and video
completion comparisons. Section~\ref{sec:conclusion} concludes the paper.

\section{Notation and Problem Formulation}
\label{sec:problem}

\subsection{Video tensor and observation model}

Let
\begin{equation}
    \X \in \R^{I\times J\times T}
\end{equation}
represent a grayscale video with $T$ frames of size $I\times J$.
The entry $x_{ijt}$ is the intensity at position $(i,j)$ in frame $t$.
Let the observation mask be
\begin{equation}
    \Mcal\in\{0,1\}^{I\times J\times T},
\end{equation}
where $m_{ijt}=1$ indicates an observed entry and $m_{ijt}=0$ indicates a
missing entry. The degraded observation is
\begin{equation}
    \Ocal=\Mcal\had\X,
    \label{eq:observation}
\end{equation}
where $\had$ denotes the Hadamard product.

For a general observation set
\begin{equation}
    \Omega=\{(i,j,t):m_{ijt}=1\},
\end{equation}
define the projection operator
\begin{equation}
    [\Pcal_{\Omega}(\X)]_{ijt}
    =
    \begin{cases}
        x_{ijt}, & (i,j,t)\in\Omega,\\
        0, & \text{otherwise}.
    \end{cases}
\end{equation}
Thus, \eqref{eq:observation} is equivalent to
$\Ocal=\Pcal_\Omega(\X)$.

We consider \emph{mode-3 tube-wise sampling}, where the same spatial
mask is applied to every frame. The observed spatial tubes form the set
\begin{equation}
    \Gamma_\Omega
    \subseteq
    \{1,\ldots,I\}\times\{1,\ldots,J\}.
\end{equation}
For tube-wise sampling,
\begin{equation}
    m_{ijt}=m_{ij},\qquad t=1,\ldots,T.
    \label{eq:tube-mask}
\end{equation}
Each observed spatial location therefore provides a complete temporal
tube, while all other tubes remain unavailable. The sampling rate is
\begin{equation}
    \mathrm{SR}
    =
    \frac{|\Gamma_\Omega|}{IJ}.
\end{equation}

The residual-learning stage uses this tube-wise mask, whereas the masked
KPD estimator in Section~\ref{sec:masked} also accommodates general
entry-wise observations.

\subsection{Kronecker-structured representation}

For vectors
\begin{equation}
    \bm a_r\in\R^I,\qquad
    \bm b_r\in\R^J,\qquad
    \bm c_r\in\R^T,
\end{equation}
define a rank-one third-order tensor
\begin{equation}
    \mathcal{K}(\bm a_r,\bm b_r,\bm c_r)
    =
    \bm a_r\od\bm b_r\od\bm c_r,
\end{equation}
whose entries satisfy
\begin{equation}
    [\bm a_r\od\bm b_r\od\bm c_r]_{ijt}
    =
    a_{ir}b_{jr}c_{tr}.
\end{equation}
Equivalently, under a consistent vectorization convention,
\begin{equation}
    \vecop\!\left(
    \bm a_r\od\bm b_r\od\bm c_r
    \right)
    =
    \bm c_r\otimes\bm b_r\otimes\bm a_r,
\end{equation}
making the Kronecker structure explicit.

A $K$-component representation is
\begin{equation}
    \X_K
    =
    \sum_{r=1}^{K}
    \bm a_r\od\bm b_r\od\bm c_r.
    \label{eq:kpd}
\end{equation}
Let
\begin{equation}
    A=[\bm a_1,\ldots,\bm a_K],\quad
    B=[\bm b_1,\ldots,\bm b_K],\quad
    C=[\bm c_1,\ldots,\bm c_K].
\end{equation}
The factor matrices $A$, $B$, and $C$ define the separable model in
\eqref{eq:kpd}, with the same parameterization as third-order CP
decomposition \cite{cheng2026clsa}. Observed-entry CP models use weighted
least squares \cite{acar2011scalable}, and standard CP-ALS alternates between
factor matrices \cite{kolda2009tensor}. Stage~I instead updates individual
components under the regularized masked objective
\eqref{eq:masked-objective} and balances their scales. The fitted tensor
provides the structural backbone for residual prediction.

We compare the storage costs of separable KPD/CP, matrix SVD, t-SVD,
and Tucker representations using the normalized storage ratio
\begin{equation}
    \rho_{\rm model}
    =
    \frac{P_{\rm model}}{IJT}\times100\%,
    \label{eq:representation-budget}
\end{equation}
where $P_{\rm model}$ is the number of stored representation parameters. Under
the implementations used here,
$P_{\rm KPD}=K(I+J+T)$,
$P_{\rm SVD}=r(IJ+T)$,
$P_{\rm tSVD}=rT(I+J+1)$ for tubal rank $r$, and
$P_{\rm Tucker}=Ir_1+Jr_2+Tr_3+r_1r_2r_3$ for multilinear rank
$(r_1,r_2,r_3)$. For t-SVD, this count includes the two truncated factor
tensors and the $r$ f-diagonal singular tubes, without an
additional reduction for Fourier conjugate symmetry.

Figure~\ref{fig:kpd-representation-efficiency} compares four representation
families on the complete development videos Akiyo, Foreman, and Mobile:
separable KPD, matrix SVD, truncated t-SVD, and Tucker-HOSVD. They represent
Kronecker-structured, matrix-spectral, tensor-spectral, and multilinear-rank
models, respectively. For each model, sequence-level fidelity is interpolated
within the storage interval shared by all three videos and then averaged.
The subsequent completion experiments fit KPD directly to masked data.

The star marks the $K=4$ backbone used in grayscale completion, whose stored parameter count is
about $0.062\%$ of the voxel count. At this
operating point, KPD and Tucker-HOSVD have comparable approximation
fidelity, while the shared storage intervals for truncated t-SVD and
matrix SVD begin at larger budgets. Spectral and multilinear
models reach higher fidelity at larger budgets. We use the compact KPD
representation as a backbone and learn its residual correction.

\begin{figure}[!t]
\centering
\subfloat[Approximation PSNR versus normalized storage.]{
    \includegraphics[width=0.46\columnwidth]{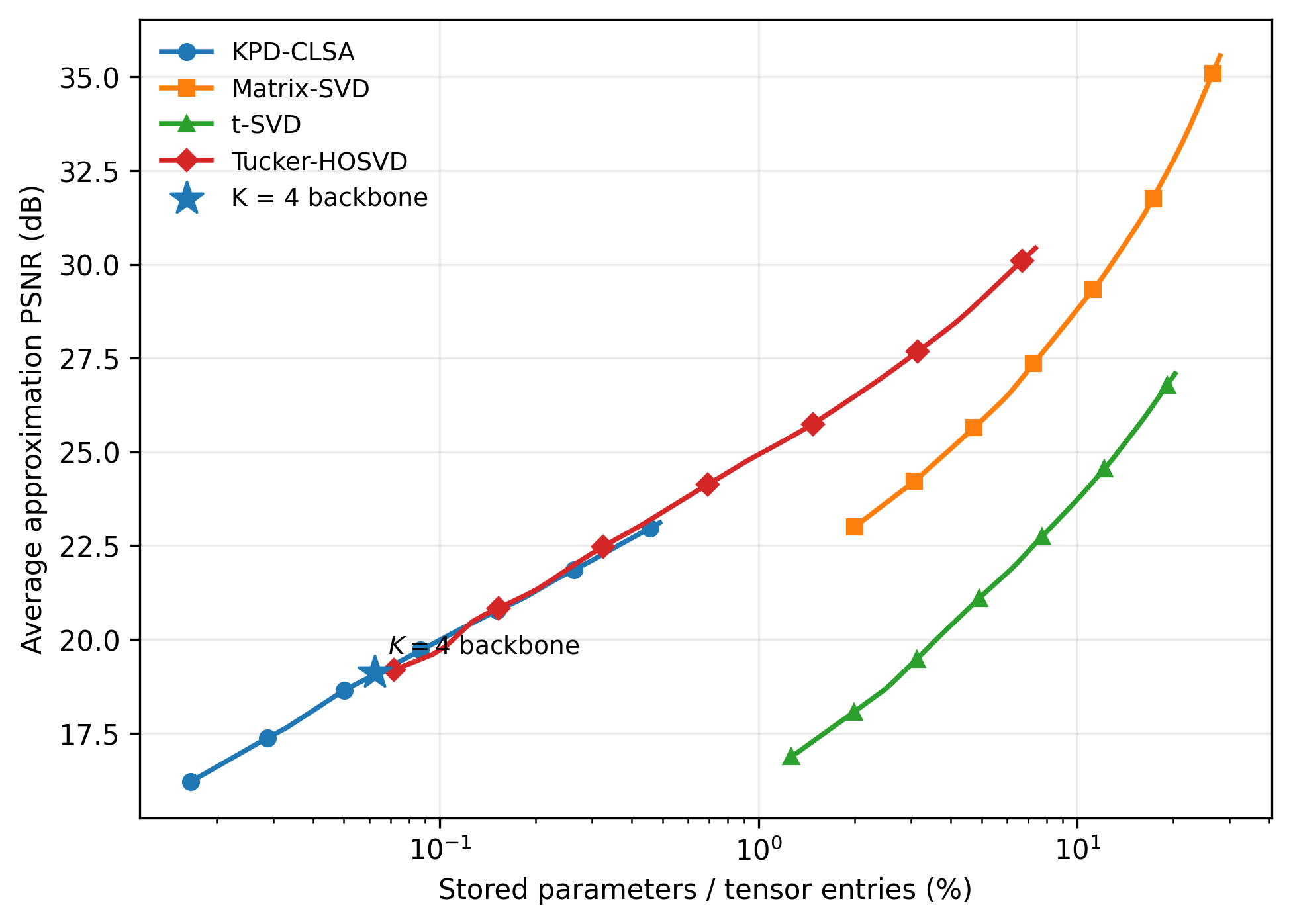}
}
\hfill
\subfloat[Approximation SSIM versus normalized storage.]{
    \includegraphics[width=0.46\columnwidth]{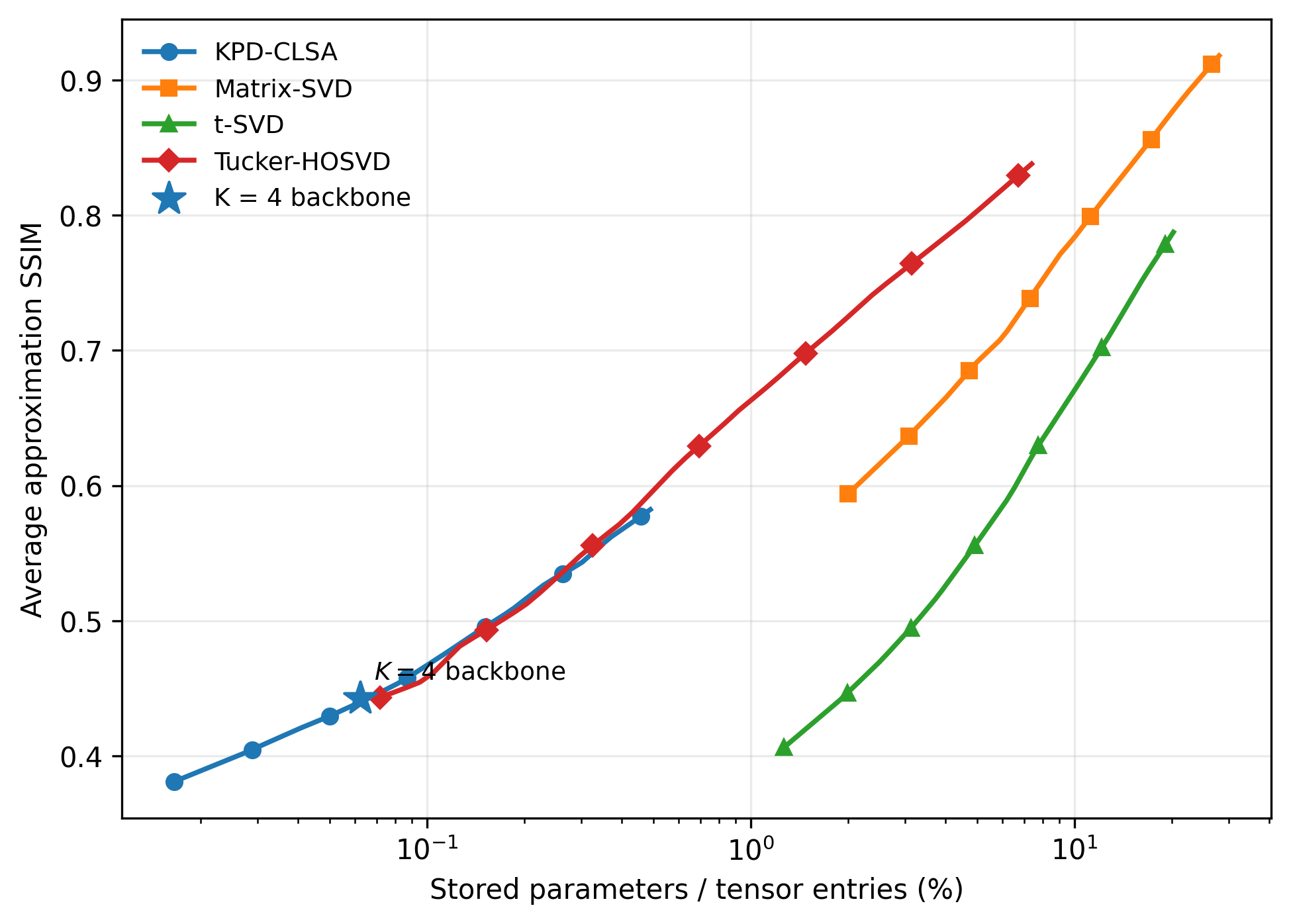}
}
\caption{Fidelity--storage tradeoff of four low-rank representation families
on the three development sequences (Akiyo, Foreman, and Mobile): (a) average
approximation PSNR and (b) average approximation SSIM versus the normalized
storage ratio in \eqref{eq:representation-budget}. For each method, the
sequence-level curves are averaged only over the storage interval shared by
all three videos. Stars mark the $K=4$ KPD-backbone operating point used in
the completion experiments.}
\label{fig:kpd-representation-efficiency}
\end{figure}

\section{Masked Kronecker Product Decomposition}
\label{sec:masked}

\begin{figure*}[!t]
\centering
\includegraphics[width=1.0\textwidth]{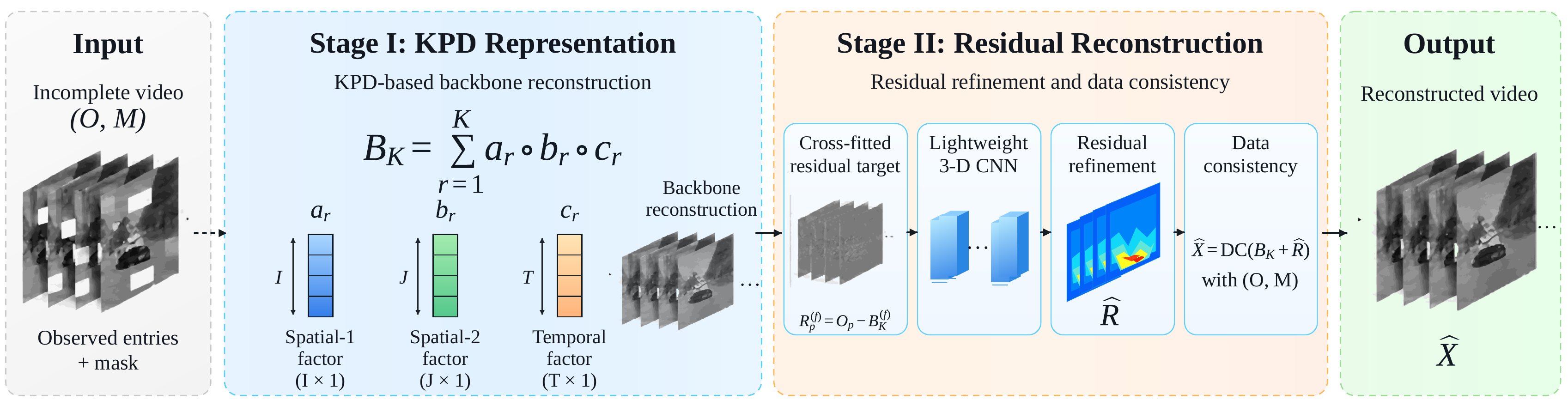}
\caption{Proposed structure--residual video-completion framework. Stage~I fits
a compact KPD backbone directly to the observed entries. Stage~II constructs
cross-fitted residual supervision, predicts the missing-tube correction with a
lightweight 3-D CNN, and applies exact data consistency to obtain
$\widehat{\X}$.}
\label{fig:framework-overview}
\end{figure*}

The workflow in Fig.~\ref{fig:framework-overview} starts with masked KPD
estimation, followed by cross-fitted residual learning. This section
derives the observed-entry objective, cyclic updates, and scale-balancing
step for backbone estimation; Section~\ref{sec:crossfit} develops the
residual predictor.

\subsection{Observed-entry objective}

We estimate the separable factors directly from observed entries. Cheng's
CLSA constructs a finite-sum KPD through component-wise least-squares
solutions of nearest-Kronecker-product subproblems \cite{cheng2026clsa}.
Here, all $K$ components are optimized under a single masked objective
and revisited in cyclic sweeps. Spatial and temporal smoothness penalties, together with Tikhonov
regularization, constrain the factors. Let $D_I$, $D_J$, and $D_T$ be the
first-order difference matrices along the corresponding modes, and set
\begin{equation}
    L_I=D_I^\top D_I,\quad
    L_J=D_J^\top D_J,\quad
    L_T=D_T^\top D_T.
\end{equation}

The masked KPD problem minimizes $\mathcal{J}(A,B,C)$ over
$\{\bm a_r,\bm b_r,\bm c_r\}_{r=1}^{K}$, where
\begin{align}
    \mathcal{J}(A,B,C)
    =&\frac{1}{2}
    \norm{
        \Mcal\had
        \left(
        \Ocal-
        \sum_{r=1}^{K}
        \bm a_r\od\bm b_r\od\bm c_r
        \right)
    }_F^2
    \nonumber\\
    &+\frac{\lambda_s}{2}
    \sum_{r=1}^{K}
    \left(
        \norm{D_I\bm a_r}_2^2
        +
        \norm{D_J\bm b_r}_2^2
    \right)
    \nonumber\\
    &+\frac{\lambda_t}{2}
    \sum_{r=1}^{K}
        \norm{D_T\bm c_r}_2^2
    \nonumber\\
    &+\frac{\lambda_0}{2}
    \sum_{r=1}^{K}
    \left(
        \norm{\bm a_r}_2^2+
        \norm{\bm b_r}_2^2+
        \norm{\bm c_r}_2^2
    \right).
    \label{eq:masked-objective}
\end{align}
The coefficients $\lambda_s\geq0$ and $\lambda_t\geq0$ penalize
spatial and temporal factor variations, and $\lambda_0>0$ imposes
Tikhonov regularization. Because the factorization admits equivalent rescalings,
these coefficients also affect the preferred factor scales.

These penalties act on factors rather than individual voxels. Since every
term in \eqref{eq:masked-objective} is nonnegative,
\begin{equation}
    \mathcal{J}(A,B,C)
    \geq
    \frac{\lambda_0}{2}
    \left(
        \norm{A}_F^2+\norm{B}_F^2+\norm{C}_F^2
    \right).
    \label{eq:coercive-bound}
\end{equation}
Therefore, $\lambda_0>0$ ensures bounded sublevel sets of $\mathcal{J}$
and prevents unbounded equivalent factor rescalings.

\subsection{Component-wise masked least-squares updates}

For the $r$th component, define the residual excluding that component as
\begin{equation}
    \Rcal^{(r)}
    =
    \Ocal
    -
    \sum_{\substack{q=1\\q\neq r}}^{K}
    \bm a_q\od\bm b_q\od\bm c_q.
    \label{eq:component-residual}
\end{equation}
With all other components fixed, the $r$th component is updated using
$\Rcal^{(r)}$.

\subsubsection{Update of the first spatial factor}

Fix $\bm b_r$ and $\bm c_r$. The subproblem for $\bm a_r$ is
\begin{align}
    \min_{\bm a_r}\quad
    &
    \frac{1}{2}
    \sum_{i,j,t}
    m_{ijt}
    \left(
        r_{ijt}^{(r)}
        -
        a_{ir}b_{jr}c_{tr}
    \right)^2
    \nonumber\\
    &+
    \frac{\lambda_s}{2}
    \norm{D_I\bm a_r}_2^2
    +
    \frac{\lambda_0}{2}
    \norm{\bm a_r}_2^2.
    \label{eq:a-subproblem}
\end{align}
Define
\begin{align}
    d^{(a)}_{i}
    &=
    \sum_{j=1}^{J}
    \sum_{t=1}^{T}
    m_{ijt}b_{jr}^2c_{tr}^2,
    \label{eq:da}
    \\
    g^{(a)}_{i}
    &=
    \sum_{j=1}^{J}
    \sum_{t=1}^{T}
    m_{ijt}r_{ijt}^{(r)}b_{jr}c_{tr}.
    \label{eq:ga}
\end{align}
Then the unique minimizer of \eqref{eq:a-subproblem} satisfies
\begin{equation}
    \left[
        \diag(\bm d^{(a)})
        +\lambda_s L_I
        +\lambda_0 I_I
    \right]
    \bm a_r
    =
    \bm g^{(a)}.
    \label{eq:a-update}
\end{equation}

\subsubsection{Update of the second spatial factor}

After updating $\bm a_r$, fix $\bm a_r$ and $\bm c_r$. Define
\begin{align}
    d^{(b)}_{j}
    &=
    \sum_{i=1}^{I}
    \sum_{t=1}^{T}
    m_{ijt}a_{ir}^2c_{tr}^2,
    \\
    g^{(b)}_{j}
    &=
    \sum_{i=1}^{I}
    \sum_{t=1}^{T}
    m_{ijt}r_{ijt}^{(r)}a_{ir}c_{tr}.
\end{align}
The corresponding update is
\begin{equation}
    \left[
        \diag(\bm d^{(b)})
        +\lambda_s L_J
        +\lambda_0 I_J
    \right]
    \bm b_r
    =
    \bm g^{(b)}.
    \label{eq:b-update}
\end{equation}

\subsubsection{Update of the temporal factor}

Finally, fix $\bm a_r$ and $\bm b_r$ and define
\begin{align}
    d^{(c)}_{t}
    &=
    \sum_{i=1}^{I}
    \sum_{j=1}^{J}
    m_{ijt}a_{ir}^2b_{jr}^2,
    \\
    g^{(c)}_{t}
    &=
    \sum_{i=1}^{I}
    \sum_{j=1}^{J}
    m_{ijt}r_{ijt}^{(r)}a_{ir}b_{jr}.
\end{align}
The temporal-factor update is
\begin{equation}
    \left[
        \diag(\bm d^{(c)})
        +\lambda_t L_T
        +\lambda_0 I_T
    \right]
    \bm c_r
    =
    \bm g^{(c)}.
    \label{eq:c-update}
\end{equation}

Equations~\eqref{eq:a-update}--\eqref{eq:c-update} constitute one
component-wise masked least-squares step. The normal equations use only
entries selected by $\Mcal$.

\subsection{Regularizer-optimal scale balancing}

For each rank-one component, consider the rescaling
\begin{equation}
    \bm a_r' = \alpha\bm a_r,\qquad
    \bm b_r' = \beta\bm b_r,\qquad
    \bm c_r' = \gamma\bm c_r,
    \qquad \alpha\beta\gamma=1,
    \label{eq:scale-orbit}
\end{equation}
which preserves the represented rank-one component. An arbitrary
normalization can change the penalty even when the data fit is unchanged.
We choose the equivalent scaling that minimizes the regularizer along this
orbit. Define
\begin{align}
    q_a &= \lambda_s\norm{D_I\bm a_r}_2^2
          +\lambda_0\norm{\bm a_r}_2^2,\\
    q_b &= \lambda_s\norm{D_J\bm b_r}_2^2
          +\lambda_0\norm{\bm b_r}_2^2,\\
    q_c &= \lambda_t\norm{D_T\bm c_r}_2^2
          +\lambda_0\norm{\bm c_r}_2^2.
\end{align}
For $q_a,q_b,q_c>0$, the constrained problem
\begin{equation}
    \min_{\alpha,\beta,\gamma>0}
    \frac12\left(q_a\alpha^2+q_b\beta^2+q_c\gamma^2\right)
    \quad\text{s.t.}\quad \alpha\beta\gamma=1
    \label{eq:balance-problem}
\end{equation}
has the closed-form solution
\begin{equation}
    g=(q_aq_bq_c)^{1/3},\qquad
    \alpha=\sqrt{\frac{g}{q_a}},\quad
    \beta=\sqrt{\frac{g}{q_b}},\quad
    \gamma=\sqrt{\frac{g}{q_c}}.
    \label{eq:balance-solution}
\end{equation}
The scaling in \eqref{eq:balance-solution} makes all three regularization
contributions equal to $g$. Since $(1,1,1)$ is feasible in
\eqref{eq:balance-problem}, balancing cannot increase the objective: the
data term remains unchanged and the penalty is minimized along the orbit.
This step is omitted if any $q$ is numerically zero.

\subsection{Cyclic masked KPD algorithm}

The algorithm cycles through the components and updates each mode. To
initialize the factors, missing entries are mean-filled and smoothed by
12 iterations of a $5\times5\times3$ uniform filter, with observed values
restored after every iteration. A KPD fit to the smoothed tensor provides
the initial factors. Subsequent sweeps minimize the masked objective
\eqref{eq:masked-objective}, with scale balancing after initialization
and each component update.

\begin{algorithm}[t]
\caption{Regularized Masked KPD by Cyclic Component Least Squares}
\label{alg:masked-kpd}
\begin{algorithmic}[1]
\Require Observation $\Ocal$, mask $\Mcal$, rank $K$, regularization
parameters $\lambda_s,\lambda_t,\lambda_0$, number of sweeps $S$.
\Ensure KPD factors $A,B,C$ and raw reconstruction $\widetilde{\X}_K$.
\State Initialize $A,B,C$.
\For{$s=1,\ldots,S$}
    \For{$r=1,\ldots,K$}
        \State Form $\Rcal^{(r)}$ by \eqref{eq:component-residual}.
        \State Update $\bm a_r$ by solving \eqref{eq:a-update}.
        \State Update $\bm b_r$ by solving \eqref{eq:b-update}.
        \State Update $\bm c_r$ by solving \eqref{eq:c-update}.
        \State Apply the regularizer-optimal scaling
        \eqref{eq:balance-solution}.
    \EndFor
\EndFor
\State $\widetilde{\X}_K\gets
\sum_{r=1}^{K}\bm a_r\od\bm b_r\od\bm c_r$.
\State \Return $A,B,C,\widetilde{\X}_K$.
\end{algorithmic}
\end{algorithm}

The result $\widetilde{\X}_K$ is the \emph{raw} reconstruction:
observed entries have not been replaced by their measurements. The fitting
residual $o_{ijt}-\widetilde{x}_{K,ijt}$ therefore remains available on
observed entries for the next stage.

\subsection{Monotonicity of the exact cyclic updates}

\begin{proposition}[Conditional optimality and monotonicity]
\label{prop:monotonicity}
Assume $\lambda_0>0$. For fixed values of all remaining factors, each
update \eqref{eq:a-update}--\eqref{eq:c-update} is the unique global minimizer
of \eqref{eq:masked-objective} with respect to the selected factor block.
Moreover, the optional scale-balancing step
\eqref{eq:balance-solution} preserves the represented rank-one component and
cannot increase the regularizer. Thus, the objective in Algorithm~\ref{alg:masked-kpd} does not increase
after an exact block or balancing step, and its sequence of values converges
to a finite limit.
\end{proposition}

\begin{proof}
Consider the $\bm a_r$ subproblem. Expanding
\eqref{eq:a-subproblem} and discarding terms independent of $\bm a_r$ yields
\begin{equation}
    \phi_a(\bm a_r)
    =
    \frac12
    \bm a_r^\top
    H_a
    \bm a_r
    -
    (\bm g^{(a)})^\top\bm a_r
    +\mathrm{const},
\end{equation}
where
\begin{equation}
    H_a
    =
    \diag(\bm d^{(a)})
    +\lambda_sL_I
    +\lambda_0 I_I.
\end{equation}
Since $\diag(\bm d^{(a)})\succeq0$, $L_I\succeq0$, and $\lambda_0>0$,
we have $H_a\succ0$. Hence $\phi_a$ is strictly convex and its unique global
minimizer is characterized by $\nabla\phi_a=0$, which gives
\eqref{eq:a-update}. The same argument applies to $\bm b_r$ and $\bm c_r$.
Each exact block replacement therefore cannot increase the full objective.
For the scale-balancing step, \eqref{eq:scale-orbit} leaves the data term
unchanged, while \eqref{eq:balance-solution} solves
\eqref{eq:balance-problem}; hence this step also cannot increase
$\mathcal{J}$. Since $\mathcal{J}\geq0$, the sequence of objective values is
monotone nonincreasing and bounded below, and thus converges.
\end{proof}

\begin{remark}
Although each factor subproblem is strictly convex, the joint objective
is nonconvex in $(A,B,C)$. Proposition~\ref{prop:monotonicity}
guarantees convergence of objective values; it does not establish global
optimality of the factor estimates. The iterates remain in a bounded
sublevel set by \eqref{eq:coercive-bound}, and their limiting estimates
may depend on initialization.
\end{remark}

\section{Cross-Fitted Self-Supervised Residual Learning}
\label{sec:crossfit}

Stage~II trains the residual predictor using observed tubes withheld from
fold-specific KPD fits. The resulting prediction errors serve as self-supervised
targets for missing-tube correction.

\subsection{Residual modeling and target mismatch}

A finite rank $K$ provides a compact representation of the structural backbone.
Let
\begin{equation}
    \widetilde{\X}_K
    =
    \sum_{r=1}^{K}
    \bm a_r\od\bm b_r\od\bm c_r
\end{equation}
be the raw KPD reconstruction. For normalized intensity data, define the
bounded structural backbone
\begin{equation}
    \Bcal_K
    =
    \Pi_{[0,1]}(\widetilde{\X}_K),
    \qquad
    \Pi_{[0,1]}(x)=\min\{1,\max\{0,x\}\},
    \label{eq:bounded-backbone}
\end{equation}
where the projection is applied elementwise. Residual learning is performed
with $\Bcal_K$ rather than an observation-clamped reconstruction. Define
\begin{equation}
    \Rcal_\star := \X-\Bcal_K,
    \qquad
    \X=\Bcal_K+\Rcal_\star.
    \label{eq:true-residual}
\end{equation}
Equation~\eqref{eq:true-residual} defines the residual to be predicted
from the backbone and the residual measurements available on observed tubes.

At an observed location, the residual can be evaluated as
\begin{equation}
    \Rcal_\Omega
    =
    \Pcal_\Omega
    \left(
        \Ocal-\Bcal_K
    \right).
    \label{eq:observed-residual}
\end{equation}
Residuals computed at tubes used to fit $\Bcal_K$ are in-sample fitting
errors. The correction needed at missing tubes is the backbone's
out-of-sample prediction error. We use cross-fitting to construct training targets for the latter
type of error. Following sample splitting
\cite{chernozhukov2018dml}, each measured target tube is withheld from the
KPD fit used to predict it.

\subsection{Tube-wise cross-fitting}

Let $\Gamma_{\mathrm{tr}}\subseteq\Gamma_\Omega$ denote the observed tubes
assigned to self-supervised training. During development, a strict internal
validation split may be formed as
\begin{equation}
    \Gamma_\Omega
    =
    \Gamma_{\mathrm{tr}}
    \,\dot\cup\,
    \Gamma_{\mathrm{val}},
    \label{eq:train-val}
\end{equation}
where $\dot\cup$ denotes disjoint union. Validation tubes are excluded
from the fold-specific KPD fits and residual-network training. Once the
configuration is selected, the final model uses all observed tubes for
cross-fitting.

Partition $\Gamma_{\mathrm{tr}}$ into $F$ disjoint folds:
\begin{equation}
    \Gamma_{\mathrm{tr}}
    =
    \dot\bigcup_{f=1}^{F}
    \Gamma_p^{(f)},
    \label{eq:fold-partition}
\end{equation}
where $\Gamma_p^{(f)}$ is the pseudo-missing set in fold $f$. Define the
corresponding context set
\begin{equation}
    \Gamma_c^{(f)}
    =
    \Gamma_{\mathrm{tr}}
    \setminus
    \Gamma_p^{(f)}.
    \label{eq:context-fold}
\end{equation}

For each fold, fit a KPD model using only its context tubes:
\begin{equation}
    \ThetaK^{(f)}
    =
    \argmin_{\ThetaK}
    \mathcal{J}_{\Gamma_c^{(f)}}(\ThetaK),
    \label{eq:fold-kpd}
\end{equation}
where $\mathcal{J}_{\Gamma_c^{(f)}}$ denotes
\eqref{eq:masked-objective} with a tube-wise mask supported only on
$\Gamma_c^{(f)}$. Let
\begin{equation}
    \widetilde{\X}_K^{(f)}
    =
    \mathcal{KPD}
    \left(
        \Ocal_{\Gamma_c^{(f)}}
    \right)
\end{equation}
denote its raw reconstruction, and define the fold-specific bounded backbone
\begin{equation}
    \Bcal_K^{(f)}
    =
    \Pi_{[0,1]}\!\left(\widetilde{\X}_K^{(f)}\right).
    \label{eq:fold-bounded-backbone}
\end{equation}

The visible context residual is
\begin{equation}
    \Ecal_c^{(f)}
    =
    \Pcal_{\Gamma_c^{(f)}}
    \left(
        \Ocal-\Bcal_K^{(f)}
    \right),
    \label{eq:context-residual}
\end{equation}
where the tube projection includes all $t=1,\ldots,T$.
The cross-fitted pseudo-missing target is
\begin{equation}
    \Rcal_p^{(f)}
    =
    \Pcal_{\Gamma_p^{(f)}}
    \left(
        \Ocal-\Bcal_K^{(f)}
    \right).
    \label{eq:crossfit-target}
\end{equation}

The tubes in $\Gamma_p^{(f)}$ belong to the original observation set
but are excluded from the fold-specific KPD fit. Equation~\eqref{eq:crossfit-target}
therefore provides measured held-out prediction errors as training targets.

\subsection{No-leakage property of the cross-fitted target}

\begin{proposition}[Cross-fitted target is a KPD hold-out error]
\label{prop:no-leakage}
For any fold $f$, suppose all data-dependent operations in the fold-specific
KPD pipeline---including preprocessing, initialization, and any adaptive
stopping or parameter selection---use only measurements in
$\Gamma_c^{(f)}$. No value in $\Gamma_p^{(f)}$ then enters the
construction of $\widetilde{\X}_K^{(f)}$ or its bounded projection
$\Bcal_K^{(f)}$. The target $\Rcal_p^{(f)}$ in
\eqref{eq:crossfit-target} is therefore an out-of-sample prediction error
of the fold-specific backbone.
\end{proposition}

\begin{proof}
Equation~\eqref{eq:context-fold} gives
$\Gamma_c^{(f)}\cap\Gamma_p^{(f)}=\varnothing$. By assumption, the
fold-specific estimator depends only on measurements in
$\Gamma_c^{(f)}$. The raw estimate $\widetilde{\X}_K^{(f)}$ and its
bounded version $\Bcal_K^{(f)}$ therefore do not depend on values in
$\Gamma_p^{(f)}$. Those values are used only after fitting, to evaluate
$\Ocal-\Bcal_K^{(f)}$ at held-out tubes, proving the claim.
\end{proof}

\begin{remark}[Relation to genuinely missing tubes]
Under missing-completely-at-random tube sampling, the fold assignment is
independent of video values and spatial content. The pseudo-missing tubes
therefore provide a held-out prediction task analogous to completion at
unobserved locations. Fold-specific and final prediction errors need not
have identical distributions, since their backbones use different numbers
of context tubes. Reserving a fraction $\nu$ of observed tubes for internal
validation leaves approximately
$(1-\nu)(1-1/F)|\Gamma_\Omega|$ context tubes per fold, compared with
$|\Gamma_\Omega|$ for the final fit. After configuration selection, we
set $\Gamma_{\mathrm{tr}}=\Gamma_\Omega$, reducing this difference to
the fold fraction $1-1/F$.
\end{remark}

\subsection{Residual network input}

For fold $f$, define the four-channel input tensor
\begin{equation}
    \Zcal^{(f)}
    =
    \left[
        \Bcal_K^{(f)},
        \Ecal_c^{(f)},
        \Mcal_c^{(f)},
        \Hcal_s\!\left(\Bcal_K^{(f)}\right)
    \right],
    \label{eq:network-input}
\end{equation}
where $\Mcal_c^{(f)}$ is the context mask and $\Hcal_s$ is a spatial
high-pass operator. In our implementation,
\begin{equation}
    \Hcal_s(\X)
    =
    \X-\operatorname{AvgPool}_{1\times5\times5}(\X).
    \label{eq:highpass}
\end{equation}
This operator extracts local spatial variation without temporal pooling.

The four channels have distinct roles:
\begin{enumerate}
    \item $\Bcal_K^{(f)}$ supplies the global structured estimate;
    \item $\Ecal_c^{(f)}$ supplies observed residual samples;
    \item $\Mcal_c^{(f)}$ distinguishes available residual samples from missing ones;
    \item $\Hcal_s(\Bcal_K^{(f)})$ highlights local spatial variation that is
    attenuated in the low-rank backbone.
\end{enumerate}

\subsection{Lightweight 3-D residual predictor}

Let $f_\theta$ denote a compact 3-D convolutional network. For an input
$\Zcal$, the forward map is
\begin{align}
    U_1
    &=
    \operatorname{GELU}
    \left(
        \operatorname{Conv3D}_{3\times3\times3}^{4\rightarrow h}(\Zcal)
    \right),
    \\
    U_2
    &=
    \operatorname{GELU}
    \left(
        \operatorname{Conv3D}_{3\times3\times3,d=2}^{h\rightarrow h}(U_1)
    \right),
    \\
    U_3
    &=
    \operatorname{GELU}
    \left(
        \operatorname{Conv3D}_{3\times3\times3}^{h\rightarrow h}(U_2)
    \right),
    \\
    \widehat{\Rcal}
    &=
    r_{\max}
    \tanh
    \left(
        \operatorname{Conv3D}_{3\times3\times3}^{h\rightarrow1}(U_3)
    \right).
    \label{eq:network}
\end{align}
Here, $d=2$ is the dilation factor, $h$ is the number of hidden
channels, and $r_{\max}>0$ bounds the residual amplitude.

The weights and bias of the final convolution are initialized to zero. Consequently,
before training,
\begin{equation}
    f_{\theta_0}(\Zcal)=0,
\end{equation}
and the network initially leaves the bounded KPD backbone unchanged.

The predictor is trained separately for each video using its observed tubes
and cross-fitted targets.

\subsection{Cross-fitted self-supervised loss}

For fold $f$, the target-dependent fitting term is computed only over the
pseudo-missing tubes $\Gamma_p^{(f)}$. Let $\rho_\delta$ denote the
smooth-$L_1$ penalty
\begin{equation}
    \rho_\delta(e)
    =
    \begin{cases}
        \dfrac{e^2}{2\delta}, & |e|<\delta,\\[4pt]
        |e|-\dfrac{\delta}{2}, & |e|\geq\delta.
    \end{cases}
    \label{eq:smoothl1}
\end{equation}
To assign greater weight to larger residuals while keeping the weights bounded, define
\begin{equation}
    w_{ijt}^{(f)}
    =
    1+
    \eta
    \min
    \left\{
        \frac{
        |r_{p,ijt}^{(f)}|
        }{\tau},
        \kappa
    \right\},
    \label{eq:detail-weight}
\end{equation}
with $\eta,\tau,\kappa>0$.

The weighted loss at a pseudo-missing target voxel is
\begin{equation}
    \ell_{ijt}^{(f)}(\theta)
    =
    w_{ijt}^{(f)}
    \rho_\delta\!\left(
        [f_\theta(\Zcal^{(f)})]_{ijt}
        - r_{p,ijt}^{(f)}
    \right).
    \label{eq:pointwise-crossfit-loss}
\end{equation}
The cross-fitted self-supervised objective is
\begin{equation}
\begin{aligned}
    \min_\theta\quad
    \mathcal{L}_{\mathrm{CF}}(\theta)
    &
    \\
    =\frac{1}{F}
    \sum_{f=1}^{F}
    \frac{1}{|\Omega_p^{(f)}|}
    &\sum_{(i,j,t)\in\Omega_p^{(f)}}
    \ell_{ijt}^{(f)}(\theta)
    +
    \frac{\mu}{F}
    \sum_{f=1}^{F}
    \frac{\norm{f_\theta(\Zcal^{(f)})}_F^2}{IJT}.
\end{aligned}
    \label{eq:crossfit-loss}
\end{equation}
Here, $\Omega_p^{(f)}$ contains the voxels in the pseudo-missing tubes of
fold $f$, and $\mu$ controls a weak residual-magnitude penalty.

At each training step, a fold and a spatiotemporal patch
containing its pseudo-missing tubes are sampled. The input tensor combines the
fold-specific backbone, observed residuals on context tubes, the context mask, and
a spatial high-pass feature computed from the backbone. The fitting term is computed
only over the held-out tubes in the patch, whereas the magnitude penalty is evaluated
over the entire predicted patch.
Held-out measurements enter only the loss targets, not the network inputs.

\subsection{Final inference and data consistency}

After residual-network training, a final raw KPD model is fitted on all
observed entries $\Omega$:
\begin{equation}
    \widetilde{\X}_K
    =
    \mathcal{KPD}(\Ocal,\Mcal),
    \qquad
    \Bcal_K=\Pi_{[0,1]}(\widetilde{\X}_K).
    \label{eq:final-kpd}
\end{equation}
The visible residual field is
\begin{equation}
    \Ecal_\Omega
    =
    \Pcal_\Omega
    \left(
        \Ocal-\Bcal_K
    \right),
    \label{eq:final-visible-residual}
\end{equation}
and the final network input is
\begin{equation}
    \Zcal
    =
    \left[
        \Bcal_K,
        \Ecal_\Omega,
        \Mcal,
        \Hcal_s(\Bcal_K)
    \right].
\end{equation}
The unconstrained refined estimate is
\begin{equation}
    \X_{\mathrm{ref}}
    =
    \Bcal_K
    +
    f_\theta(\Zcal).
    \label{eq:refined}
\end{equation}
The data-consistent output is
\begin{equation}
    \widehat{\X}
    =
    \Mcal\had\Ocal
    +
    (\bm 1-\Mcal)
    \had
    \Pi_{[0,1]}
    \left(
        \X_{\mathrm{ref}}
    \right).
    \label{eq:data-consistency}
\end{equation}

\begin{proposition}[Exact data consistency]
\label{prop:data-consistency}
The output of \eqref{eq:data-consistency} satisfies
\begin{equation}
    \Pcal_\Omega(\widehat{\X})
    =
    \Pcal_\Omega(\Ocal).
\end{equation}
Hence the final output retains every observed entry exactly.
\end{proposition}

\begin{proof}
For $(i,j,t)\in\Omega$, we have $m_{ijt}=1$, so
\eqref{eq:data-consistency} yields $\widehat{x}_{ijt}=o_{ijt}$.
This holds for every observed entry.
\end{proof}

\subsection{Training and inference}

\begin{algorithm}[t]
\caption{Cross-fitted self-supervised residual completion}
\label{alg:crossfit}
\begin{algorithmic}[1]
\Require Observation $\Ocal$, tube mask $\Mcal$, KPD rank $K$, number of
cross-fitting folds $F$.
\Ensure Reconstructed video $\widehat{\X}$.
\State Obtain the observed tube set $\Gamma_\Omega$.
\State Optionally split
$\Gamma_\Omega=\Gamma_{\mathrm{tr}}\dot\cup\Gamma_{\mathrm{val}}$.
\State Partition
$\Gamma_{\mathrm{tr}}=
\dot\bigcup_{f=1}^{F}\Gamma_p^{(f)}$.
\For{$f=1,\ldots,F$}
    \State $\Gamma_c^{(f)}
    \gets
    \Gamma_{\mathrm{tr}}\setminus\Gamma_p^{(f)}$.
    \State Fit $\widetilde{\X}_K^{(f)}$ using
    Algorithm~\ref{alg:masked-kpd} on $\Gamma_c^{(f)}$ only, with every
    data-dependent initialization step restricted to the same context.
    \State Set $\Bcal_K^{(f)}\gets\Pi_{[0,1]}(\widetilde{\X}_K^{(f)})$.
    \State Construct context residual $\Ecal_c^{(f)}$ by
    \eqref{eq:context-residual}.
    \State Construct pseudo-missing target $\Rcal_p^{(f)}$ by
    \eqref{eq:crossfit-target}.
    \State Form $\Zcal^{(f)}$ by \eqref{eq:network-input}.
\EndFor
\State Train $f_\theta$ using \eqref{eq:crossfit-loss}.
\State Fit the final raw KPD reconstruction $\widetilde{\X}_K$ on all observed tubes and set $\Bcal_K\gets\Pi_{[0,1]}(\widetilde{\X}_K)$.
\State Form $\Ecal_\Omega$ using \eqref{eq:final-visible-residual}.
\State $\widehat{\Rcal}\gets
f_\theta([\Bcal_K,\Ecal_\Omega,\Mcal,
\Hcal_s(\Bcal_K)])$.
\State Apply data consistency by \eqref{eq:data-consistency}.
\State \Return $\widehat{\X}$.
\end{algorithmic}
\end{algorithm}

The full estimator combines the structural backbone with the residual
prediction before applying data consistency. Denoting the data-consistency operator by
$\operatorname{DC}(\cdot)$ gives
\begin{equation}
    \widehat{\X}
    =
    \operatorname{DC}\!\left(
        \Bcal_K + f_\theta(\Zcal)
    \right).
    \label{eq:framework-box}
\end{equation}
The backbone models low-dimensional spatiotemporal structure, and the
network predicts a correction from the observed residual context.

\subsection{Computational complexity}
\label{sec:complexity}

Let $N=IJT$ and $N_\Omega=|\Omega|$ be the total and observed voxel
counts. Contractions in \eqref{eq:da}--\eqref{eq:c-update} require
$\mathcal{O}(N_\Omega)$ operations per component when computed over
observed entries. With banded solvers, the three positive-definite
tridiagonal systems require $\mathcal{O}(I+J+T)$ operations per component.
An observation-indexed implementation of $S$ sweeps therefore costs
\begin{equation}
    \mathcal{O}\!\left(SK(N_\Omega+I+J+T)\right).
    \label{eq:kpd-complexity}
\end{equation}

Our implementation stores dense $I\times J\times T$ masks and residual
arrays and evaluates full-tensor contractions. For $S$ sweeps, its cost is
\begin{equation}
    \mathcal{O}\!\left(SK(N+I+J+T)\right),
    \label{eq:kpd-dense-complexity}
\end{equation}
which reduces to $\mathcal{O}(SKN)$ when $N$ dominates the factor dimensions.
These complexity bounds exclude warm-start initialization. The implementation employs
banded tridiagonal solvers; replacing them with dense linear solvers can add
$\mathcal{O}(SK(I^3+J^3+T^3))$ operations in the worst case.

Cross-fitting requires $F$ auxiliary KPD fits with different context masks.
Their total cyclic-refinement cost under dense contractions is
$\mathcal{O}(FSK(N+I+J+T))$, followed by a final fit on all observed
tubes. The auxiliary fits incur approximately $F$ times the cost of a single KPD
fit. For fixed $3\times3\times3$ kernels and hidden width $h$, one
residual-network forward pass costs $\mathcal{O}(Nh^2)$.
The CNN training cost depends on the number of optimization steps,
batch size, and patch volume. The experiments report end-to-end runtimes.


\section{Experiments}
\label{sec:experiments}

\subsection{Protocol and implementation}

We evaluate the proposed method on 23 grayscale video sequences
and eight selected DAVIS 2017 RGB sequences
\cite{ponttuset2017davis}. The grayscale benchmark contains
17 videos of size $144\times176\times100$ and six videos of
size $288\times352\times50$, with sampling rates
$\mathrm{SR}\in\{0.1,0.2,0.3,0.4,0.5\}$. The RGB benchmark
includes Bear, Drift-chicane, Cows, Dog, Camel, Kite-surf,
Parkour, and Breakdance, selected to cover different visual
content and motion patterns. Each RGB sequence contains
50 frames resized to $360\times640$ pixels, with
$\mathrm{SR}\in\{0.2,0.3,0.4\}$. Both benchmarks use
mode-3 tube-wise missingness, with the
same spatial sampling mask applied across all frames.
For RGB videos, the mask is also shared across color
channels, which are reconstructed separately.

The grayscale configuration is selected using Akiyo,
Foreman, and Mobile at $\mathrm{SR}=0.2$ with an 80/20
split of their observed tubes for internal training and
validation. Validation MAE guides configuration selection;
the selected settings are then used for the 23-video benchmark.
For RGB experiments, Blackswan at $\mathrm{SR}=0.3$
serves as the development sequence. Its missing-region PSNR,
evaluated against the available reference, determines the
configuration applied unchanged to the eight reported
sequences and three sampling rates. These
settings specify only the hyperparameters. Each video's structural
backbone and residual CNN are fitted from its own observed
tubes. Table~\ref{tab:reference-config} lists the configurations, and
Supplementary Appendix~A details the RGB selection study.
\begin{table}[!ht]
\centering
\caption{Fixed hyperparameter configurations for
the grayscale and DAVIS RGB benchmarks.
Parameters marked as ``same'' use the grayscale setting.}
\label{tab:reference-config}
\footnotesize
\setlength{\tabcolsep}{4pt}
\begin{tabular}{lcc}
\toprule
Parameter & 23-video grayscale & DAVIS 360p RGB\\
\midrule
KPD rank $K$ & $4$ & $12$\\
Masked KPD sweeps $S$ & $12$ & same\\
Spatial penalty $\lambda_s$ & $5\times10^{-3}$ & same\\
Temporal penalty $\lambda_t$ & $5\times10^{-2}$ & same\\
Tikhonov $\lambda_0$ & $10^{-5}$ & same\\
Cross-fitting folds $F$ & $5$ & same\\
Conv3D kernel $(T,H,W)$ & $3\times3\times3$ & same\\
Hidden channels $h$ & $12$ & $32$\\
Residual bound $r_{\max}$ & $0.25$ & same\\
Training patch $(T,H,W)$ & $16\times48\times48$ & $16\times96\times96$\\
Batch size & $2$ & same\\
Optimization steps & $1200$ & $3600$\\
Learning rate & $10^{-3}$ & same\\
Smooth-$L_1$ threshold $\delta$ & $0.02$ & same\\
Magnitude weight $\eta$ & $2$ & same\\
Magnitude scale $\tau$ & $0.05$ & same\\
Magnitude cap $\kappa$ & $3$ & same\\
Residual penalty $\mu$ & $10^{-4}$ & same\\
\bottomrule
\end{tabular}
\end{table}

The residual CNN is trained using Adam with a learning
rate of $10^{-3}$. Automatic mixed precision is used
on CUDA, and the gradient norm is clipped at $3.0$.
The residual predictor is trained separately for each
video using only its observed measurements. The masked
KPD initialization and factor-scale balancing follow
Section~\ref{sec:masked}.

\begin{table*}[!t]
\centering
\caption{Residual-supervision ablation at $\mathrm{SR}=0.2$ and $K=4$.
Sequence entries are full-video PSNR/frame-wise full-frame SSIM; validation
MAE is averaged over the three development sequences. Best and second-best
results are bold and underlined.}
\label{tab:method-ablation}
\scriptsize
\setlength{\tabcolsep}{4.2pt}
\begin{tabular}{lcccccc}
\toprule
Method & Akiyo & Foreman & Mobile & Val. MAE $\downarrow$ & Avg. PSNR $\uparrow$ & Avg. SSIM $\uparrow$\\
\midrule
KPD
& 22.989/0.5589 & 19.348/0.4286 & 16.002/0.3591 & 0.092783 & 19.446 & 0.4488\\
Simple
& 24.801/0.6565 & 25.462/0.7554 & 19.050/0.5687 & 0.056637 & 23.104 & 0.6602\\
Direct
& \underline{26.280/0.7282} & \underline{25.539/0.7670} & \textbf{20.464/0.6998}
& \underline{0.049610} & \underline{24.095} & \underline{0.7317}\\
CrossFit
& \textbf{26.972/0.7465} & \textbf{25.771/0.7859} & \underline{20.397/0.6936}
& \textbf{0.047503} & \textbf{24.380} & \textbf{0.7420}\\
\bottomrule
\end{tabular}
\end{table*}

\begin{figure*}[!t]
\centering
\includegraphics[width=0.9\textwidth]{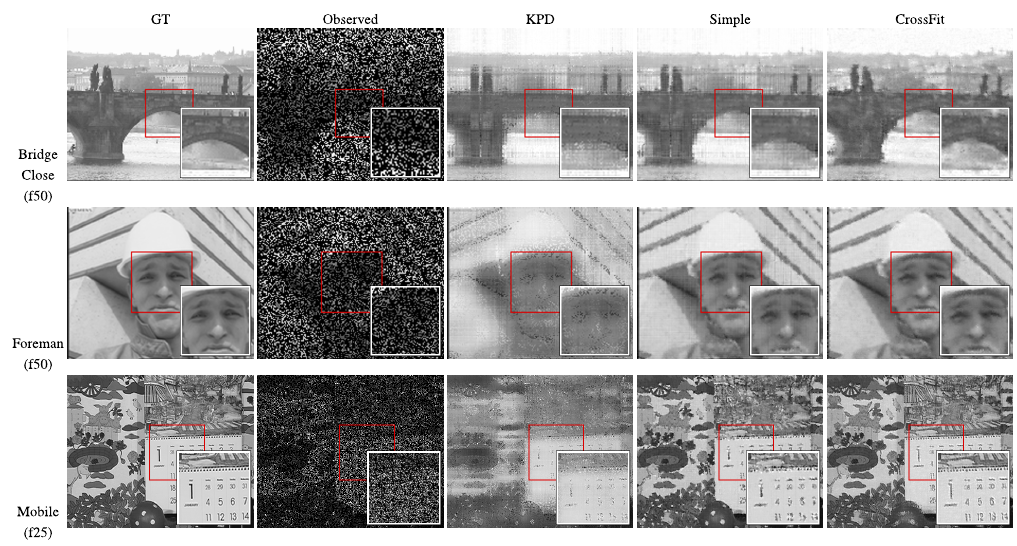}
\caption{Stage-wise visual ablation at $\mathrm{SR}=0.3$: ground truth,
observed input, KPD backbone, simple residual propagation, and CrossFit
reconstruction. A common ROI is enlarged within each sequence.}
\label{fig:stagewise-qualitative-sr03}
\end{figure*}

The grayscale experiments are conducted on an AMD
Ryzen 9 9950X CPU, an NVIDIA RTX 4090 GPU (24 GB),
and 128 GB RAM, using PyTorch 2.0.1, CUDA 11.8,
and MATLAB R2022a. The RGB experiments use the same
CPU and RAM with an NVIDIA RTX 5080 GPU (16 GB),
PyTorch 2.7.1, CUDA 12.8, and MATLAB. KBR, TNN,
and SALTS run on CPU workers, while the
remaining methods use their respective Python backends.
Runtime results are reported separately for the two
hardware configurations. The source code of the proposed method and the corresponding experimental scripts are publicly available at \url{https://github.com/QQ-SDU/KPD-Video-Completion}.

We use peak signal-to-noise ratio (PSNR), structural
similarity (SSIM) \cite{wang2004ssim}, temporal-gradient
mean absolute error (T-MAE), and wall-clock runtime
for evaluation. Grayscale PSNR is measured over the
full video, whereas RGB PSNR is computed over missing
entries. SSIM is computed over each full frame and then
averaged across frames. T-MAE evaluates the accuracy of temporal
intensity changes at missing tubes. Let $\mathcal M$
denote the set of missing spatial--channel indices,
$T$ the number of frames, and
$\Delta_t X_p=X_{p,t+1}-X_{p,t}$. T-MAE is defined as
\begin{equation}
\label{eq:temporal-mae-main}
\mathrm{TMAE}_{\mathrm{miss}}=
\frac{\displaystyle\sum_{t=1}^{T-1}
\sum_{p\in\mathcal M}
\left|\Delta_t\widehat X_p-\Delta_t X_p\right|}
{(T-1)|\mathcal M|},
\end{equation}
where $X$ and $\widehat X$ denote the reference and
reconstructed videos, respectively. All reconstruction
metrics are computed after enforcing data consistency.
Higher PSNR and SSIM indicate better reconstruction
quality, while lower T-MAE indicates more accurate
reconstruction of frame-to-frame intensity changes.

\begin{table}[!t]
\centering
\caption{Residual-network input ablation on the three development sequences
at $\mathrm{SR}=0.2$. PSNR is computed over the full reconstructed video and
SSIM is averaged frame-wise over the full frames.}
\label{tab:input-ablation}
\scriptsize
\setlength{\tabcolsep}{3.5pt}
\begin{tabular}{lcccc}
\toprule
Input & Val. MAE $\downarrow$ & PSNR $\uparrow$ & SSIM $\uparrow$ & $\Delta$PSNR\\
\midrule
B    & 0.092944 & 19.360 & 0.4561 & --\\
BR   & 0.060328 & 22.354 & 0.6447 & +2.994\\
BRM  & \underline{0.050618} & \underline{23.944} & \underline{0.7073} & +1.590\\
FULL & \textbf{0.047503} & \textbf{24.380} & \textbf{0.7420} & +0.436\\
\bottomrule
\end{tabular}
\end{table}

\subsection{Ablation of residual supervision}
\begin{table*}[!t]
\centering
\caption{Average PSNR/SSIM on the 23-video benchmark at five tube sampling
rates, with mean wall-clock time per sequence. Best and second-best values are
bold and underlined.}
\label{tab:allrate-comparison}
\scriptsize
\setlength{\tabcolsep}{2.7pt}
\resizebox{\textwidth}{!}{%
\begin{tabular}{l*{5}{cc}c}
\toprule
\multirow{2}{*}{Method}
& \multicolumn{2}{c}{SR=0.1}
& \multicolumn{2}{c}{SR=0.2}
& \multicolumn{2}{c}{SR=0.3}
& \multicolumn{2}{c}{SR=0.4}
& \multicolumn{2}{c}{SR=0.5}
& \multirow{2}{*}{Avg. Time (s) $\downarrow$}\\
\cmidrule(lr){2-3}\cmidrule(lr){4-5}\cmidrule(lr){6-7}\cmidrule(lr){8-9}\cmidrule(lr){10-11}
& PSNR $\uparrow$ & SSIM $\uparrow$
& PSNR $\uparrow$ & SSIM $\uparrow$
& PSNR $\uparrow$ & SSIM $\uparrow$
& PSNR $\uparrow$ & SSIM $\uparrow$
& PSNR $\uparrow$ & SSIM $\uparrow$ & \\
\midrule
DRO-TFF\,\cite{li2025drotff} & 15.091 & 0.1628 & 18.804 & 0.3365 & 22.227 & 0.5120 & 24.787 & 0.6381 & 27.010 & 0.7389 & 7.79 \\
HLRTF\,\cite{luo2022hlrtf} & 20.887 & 0.4819 & 23.626 & 0.6408 & 25.565 & 0.7290 & 26.987 & 0.7757 & 28.425 & 0.8209 & \textbf{7.17} \\
KBR\,\cite{xie2018kbr} & 17.867 & 0.3324 & 21.009 & 0.5127 & 24.428 & 0.6759 & 27.880 & 0.8134 & 30.010 & 0.8993 & 154.21 \\
LRTFR\,\cite{luo2024lrtfr} & 21.793 & 0.5238 & 24.549 & 0.6893 & 26.419 & 0.7784 & 28.011 & 0.8378 & 29.597 & 0.8807 & \underline{7.38} \\
OTLRM\,\cite{wang2025otlrm} & 19.805 & 0.4408 & 24.052 & 0.6616 & 25.194 & 0.7195 & 26.664 & 0.7809 & 28.653 & 0.8436 & 80.91 \\
TNN\,\cite{lu2018tnn} & 19.847 & 0.4036 & 22.619 & 0.5796 & 24.811 & 0.7010 & 26.742 & 0.7870 & 28.669 & 0.8517 & 62.90 \\
SALTS\,\cite{wu2024salts} & 19.567 & 0.4412 & 24.310 & 0.7054 & 26.917 & 0.8183 & 28.974 & 0.8737 & 30.306 & \underline{0.9238} & 1190.95 \\
S$^2$NTNN\,\cite{luo2022s2ntnn} & \textbf{23.975} & \textbf{0.7215} & \textbf{25.991} & \textbf{0.8102} & \underline{27.263} & \underline{0.8409} & \underline{29.048} & \underline{0.8788} & \underline{30.360} & 0.9230 & 52.84 \\
\textbf{Ours} & \underline{22.288} & \underline{0.5577} & \underline{25.349} & \underline{0.7428} & \textbf{27.577} & \textbf{0.8513} & \textbf{29.734} & \textbf{0.8811} & \textbf{30.889} & \textbf{0.9245} & 49.07 \\
\bottomrule
\end{tabular}%
}
\end{table*}

\begin{figure*}[!t]
\centering
\includegraphics[width=1\textwidth]{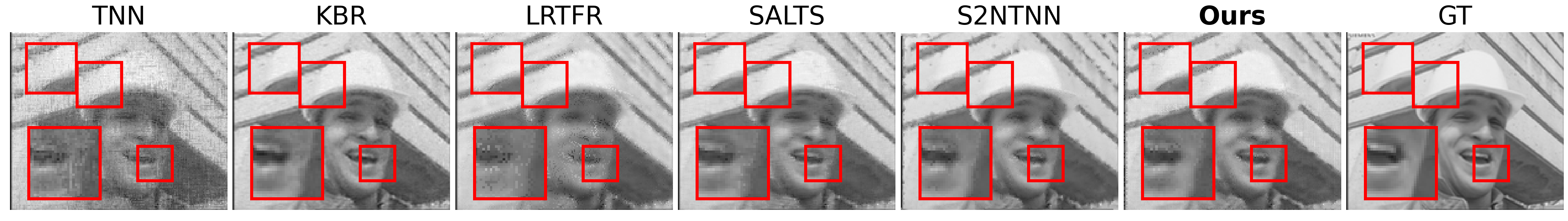}\\[-0.5mm]
\includegraphics[width=1\textwidth]{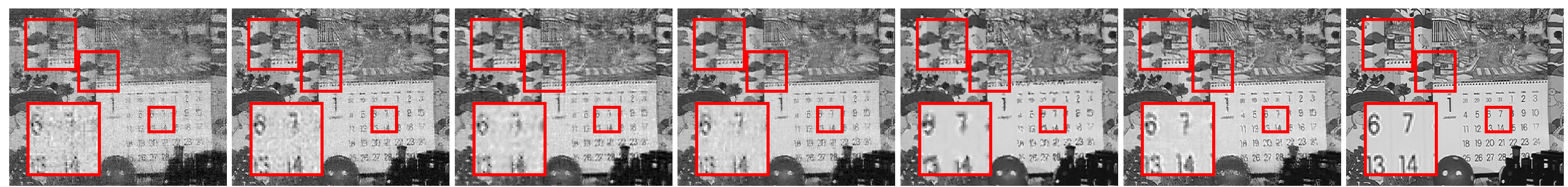}
\caption{Visual comparison at $\mathrm{SR}=0.3$ on Foreman and Mobile. Common frames and ROIs show detail recovery in both sequences.}
\label{fig:qualitative-sr03}
\end{figure*}

We isolate the effect of residual supervision by comparing the KPD backbone,
Gaussian residual propagation (Simple), direct residual learning (Direct),
and cross-fitted residual learning (CrossFit) in Table~\ref{tab:method-ablation}.
Simple interpolates visible KPD residuals spatially and adds the interpolated
field to the backbone. Its Gaussian width
$\sigma\in\{0.75,1,1.5,2,3\}$ and residual scale
$\alpha\in\{0.5,0.75,1,1.25\}$ are selected using MAE on an
internal calibration split containing 15\% of the development-training
tubes. Direct and CrossFit share the same CNN architecture and optimization
budget. In Direct, the backbone is fitted using all development-training
observations before the pseudo-holdout step, so the CNN targets measure
in-sample residuals; CrossFit uses held-out backbone errors.

As shown in Table~\ref{tab:method-ablation}, CrossFit increases the
three-video mean PSNR from $24.095$ to $24.380$ dB
($+0.285$ dB) and SSIM from $0.7317$ to $0.7420$
($+0.0103$), while reducing validation MAE from
$0.049610$ to $0.047503$. CrossFit performs better on
Akiyo and Foreman; Direct obtains slightly higher PSNR
and SSIM on Mobile.

\subsection{Ablation of residual-network inputs}

We assess the contribution of the input channels in
\eqref{eq:network-input}. The four configurations add information
incrementally: B uses $\Bcal_K$; BR adds $\Ecal_c$; BRM also supplies
$\Mcal_c$; FULL uses all four channels. Excluded channels are set to zero
to keep network capacity unchanged. Table~\ref{tab:input-ablation} reports the results.

The visible-residual channel produces the largest PSNR
gain, raising the three-video mean from $19.360$ dB
for B to $22.354$ dB for BR. The observation mask raises
mean PSNR to $23.944$ dB, and the high-pass channel
increases it further to $24.380$ dB. FULL also achieves
the highest SSIM and lowest validation MAE among the four
configurations (Table~\ref{tab:input-ablation}).

\subsection{Stage-wise visual ablation}

Figure~\ref{fig:stagewise-qualitative-sr03} shows the stage-wise
reconstructions. The KPD backbone recovers the main scene geometry, and
the residual-refinement stage recovers additional edge and texture details
in the enlarged regions.

\subsection{Sensitivity analysis}

Using the three grayscale development sequences, we examine sensitivity
to KPD rank, the number of cross-fitting folds, optimization steps, and
spatial and temporal regularization. Figure~\ref{fig:sensitivity-overview} shows
how the reconstruction changes with these settings and identifies the
selected operating points.

\begin{figure}[!htbp]
\centering
\includegraphics[width=1.0\columnwidth]{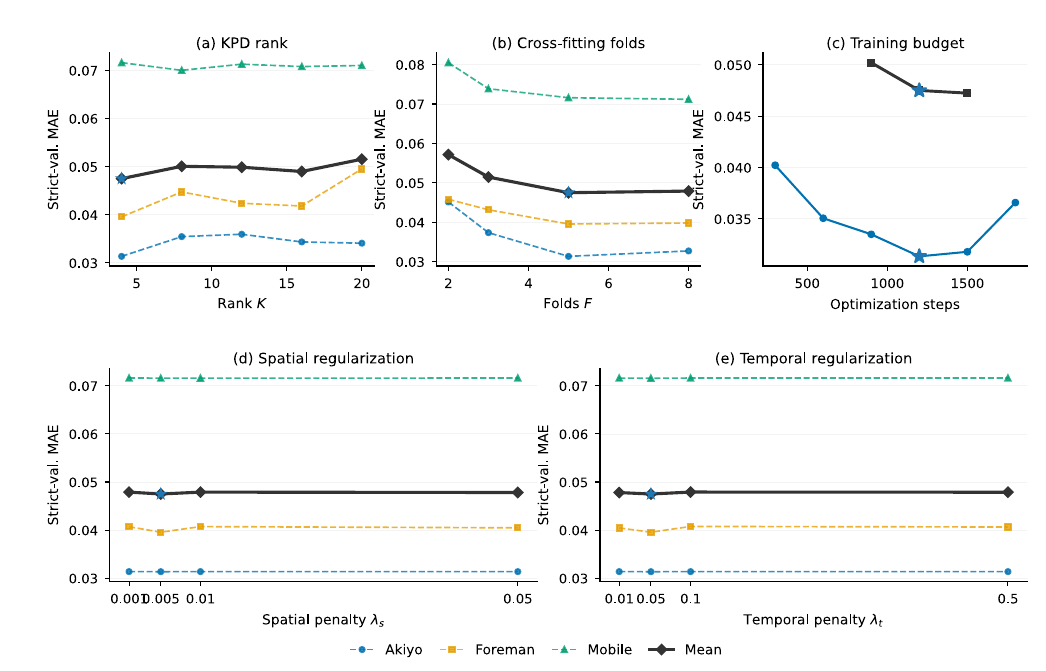}
\caption{Sensitivity on the three grayscale development sequences at
$\mathrm{SR}=0.2$: (a) KPD rank, (b) cross-fitting folds, (c) optimization
steps, (d) spatial regularization, and (e) temporal regularization.
Stars mark the selected settings.}
\label{fig:sensitivity-overview}
\end{figure}

\begin{figure*}[!t]
\centering
\includegraphics[width=1\textwidth]{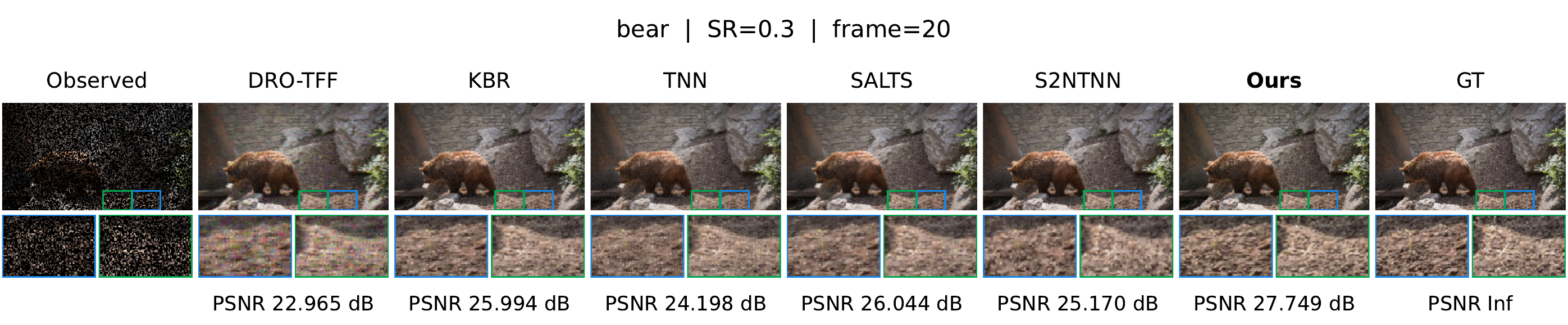}\\[0.6mm]
\includegraphics[width=1\textwidth]{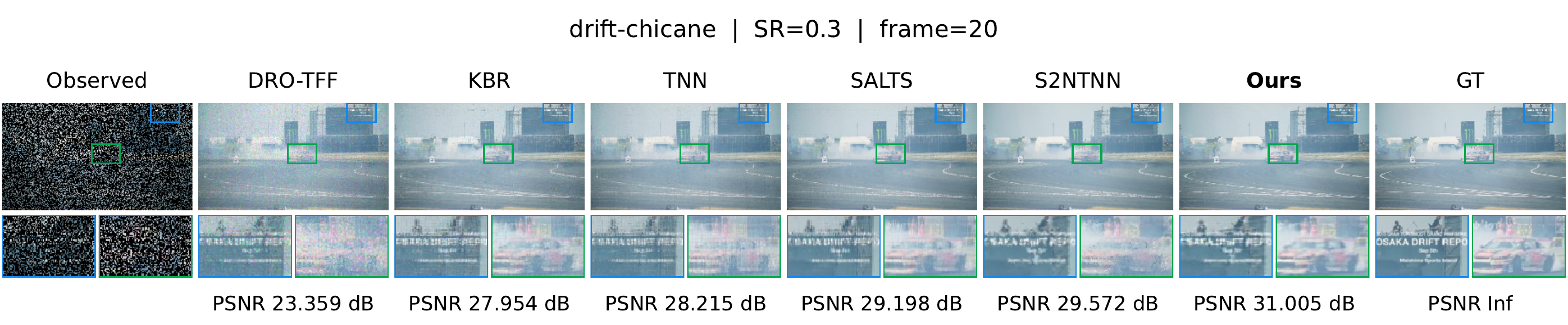}\\[0.6mm]
\includegraphics[width=1\textwidth]{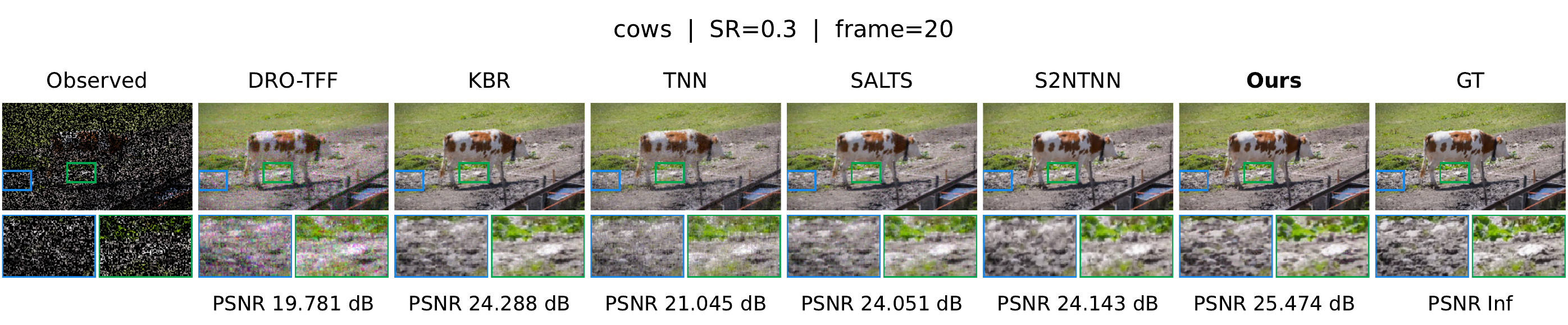}
\caption{Qualitative RGB comparison on DAVIS 2017 at $\mathrm{SR}=0.3$ for three representative sequences: Bear, Drift-chicane, and Cows. Each row uses fixed ROIs across methods. PSNR values are computed over the missing tubes of the full sequence.}
\label{fig:davis-rgb-visual}
\end{figure*}
\label{sec:qual-rgb}
\subsection{Comparison on the 23-video benchmark}

Table~\ref{tab:allrate-comparison} reports PSNR, SSIM, and runtime for
nine methods on 23 grayscale videos. Figure~\ref{fig:final-comparison}
shows PSNR, SSIM, and temporal-gradient MAE across sampling
rates for this benchmark and the eight selected DAVIS 360p RGB sequences. Additional
runtime measurements appear in the supplementary material.
\begin{figure}[!h]
\centering
\includegraphics[width=1.0\columnwidth]{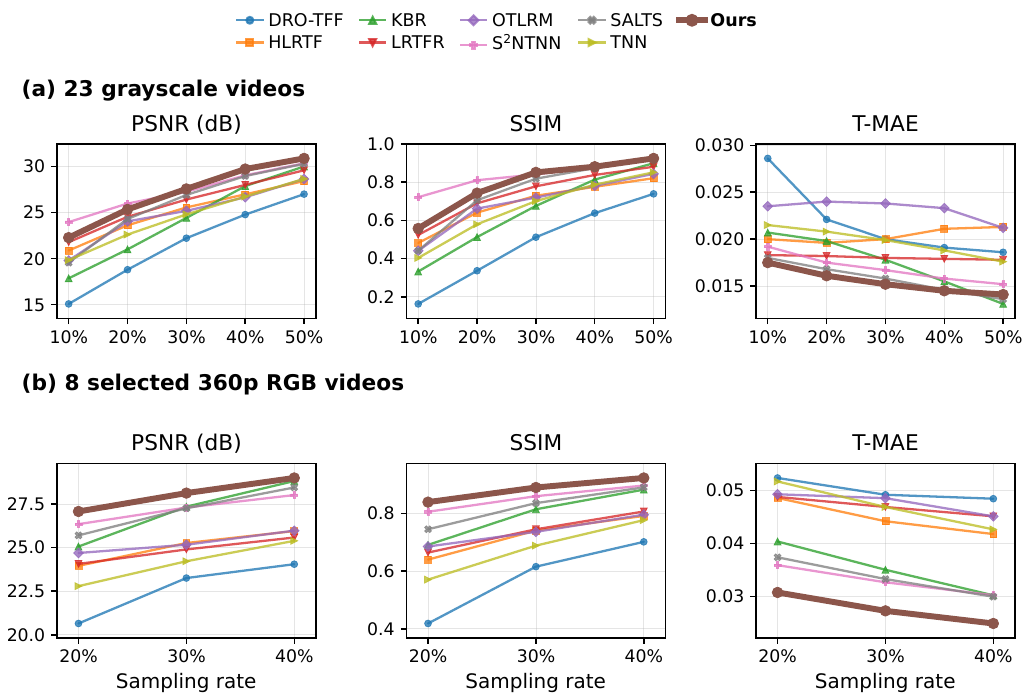}
\caption{Mean PSNR, SSIM, and missing-tube temporal-gradient MAE (T-MAE) versus sampling rate for (a) the 23-video grayscale benchmark and (b) eight selected DAVIS 360p RGB sequences. Thick lines identify Ours. T-MAE measures the error in frame-to-frame intensity differences at missing tubes (\ref{eq:temporal-mae-main}). DAVIS curves use per-sequence measurements; grayscale curves use the recorded benchmark statistics.}
\label{fig:final-comparison}
\end{figure}

Table~\ref{tab:allrate-comparison} shows a change in method ranking with
sampling density. S$^2$NTNN achieves the highest mean PSNR and SSIM at
$\mathrm{SR}=0.1$ and $0.2$. Ours achieves the highest means on both
metrics at $\mathrm{SR}=0.3$, $0.4$, and $0.5$, exceeding the respective
second-best methods by $0.314$, $0.686$, and $0.529$ dB in PSNR and by
$0.0104$, $0.0023$, and $0.0007$ in SSIM. Mean runtime is $49.07$ s
for Ours and $7.17$ s for HLRTF, the fastest method in the table.
Figure~\ref{fig:qualitative-sr03} presents the Foreman and Mobile visual
comparisons.

\subsection{DAVIS 2017 RGB completion at 360p}
\label{sec:davis-rgb}

Table~\ref{tab:davis-rgb-main} reports results for the three videos in
Fig.~\ref{fig:davis-rgb-visual} and averages over the eight
selected DAVIS RGB sequences. The configuration selected on Blackswan
is used at each sampling rate. Mean
missing-region PSNR reaches $27.07$, $28.12$, and $28.99$
dB at $\mathrm{SR}=0.2$, $0.3$, and $0.4$, respectively,
exceeding the best competing mean at each rate by $0.73$,
$0.78$, and $0.18$ dB. Ours also attains the highest mean
full-frame SSIM at these rates ($0.838$, $0.889$, and
$0.922$). Figure~\ref{fig:final-comparison} shows the
corresponding PSNR, SSIM, and temporal-gradient trends.

For missing-tube T-MAE, Ours obtains the lowest mean at
each rate ($0.03079$, $0.02732$, and $0.02495$). It also
has the lowest T-MAE in $19$ of the $24$ video--rate
cases. The supplementary material reports the
per-sequence measurements used to compute these averages.

\subsection{Qualitative RGB comparisons}
\begin{table*}[!t]
\centering
\caption{DAVIS 360p RGB reconstruction results.
(a) Three representative sequences at $\mathrm{SR}=0.3$;
(b) mean results over eight selected sequences.
Each entry reports missing-region PSNR (dB) / full-frame
SSIM on the first line and missing-tube temporal-gradient
MAE on the second line. Best and second-best results
for each metric are bold and underlined, respectively.}
\label{tab:davis-rgb-main}

\begingroup
\fontsize{7.6}{9.2}\selectfont
\setlength{\tabcolsep}{3pt}
\renewcommand{\arraystretch}{1.15}

\begin{tabular*}{\textwidth}{
@{\extracolsep{\fill}}lcccccc@{}}
\toprule
\multirow{2}{*}{Method}
& \multicolumn{3}{c}{(a) Representative sequences ($\mathrm{SR}=0.3$)}
& \multicolumn{3}{c}{(b) Mean over eight sequences} \\
\cmidrule(lr){2-4}\cmidrule(l){5-7}
& Bear & Drift-chicane & Cows
& $\mathrm{SR}=0.2$ & $\mathrm{SR}=0.3$ & $\mathrm{SR}=0.4$ \\
\midrule

DRO-TFF\,\cite{li2025drotff}
& \shortstack[c]{22.96 / 0.652\\0.05461}
& \shortstack[c]{23.36 / 0.428\\0.00883}
& \shortstack[c]{19.78 / 0.527\\0.06434}
& \shortstack[c]{20.65 / 0.419\\0.05233}
& \shortstack[c]{23.25 / 0.615\\0.04917}
& \shortstack[c]{24.04 / 0.701\\0.04842} \\

HLRTF\,\cite{luo2022hlrtf}
& \shortstack[c]{24.11 / 0.733\\0.05164}
& \shortstack[c]{29.51 / 0.846\\0.00777}
& \shortstack[c]{21.86 / 0.640\\0.06312}
& \shortstack[c]{23.95 / 0.640\\0.04856}
& \shortstack[c]{25.26 / 0.740\\0.04419}
& \shortstack[c]{25.94 / 0.791\\0.04177} \\

KBR\,\cite{xie2018kbr}
& \shortstack[c]{25.99 / 0.811\\0.04833}
& \shortstack[c]{27.95 / 0.750\\0.00707}
& \shortstack[c]{\underline{24.29} / 0.765\\0.05631}
& \shortstack[c]{25.06 / 0.691\\0.04038}
& \shortstack[c]{\underline{27.34} / 0.814\\0.03505}
& \shortstack[c]{\underline{28.81} / 0.882\\0.03024} \\

LRTFR\,\cite{luo2024lrtfr}
& \shortstack[c]{24.01 / 0.732\\0.05276}
& \shortstack[c]{28.31 / 0.792\\0.00766}
& \shortstack[c]{23.40 / 0.724\\0.05933}
& \shortstack[c]{24.06 / 0.664\\0.04876}
& \shortstack[c]{24.89 / 0.744\\0.04687}
& \shortstack[c]{25.58 / 0.806\\0.04511} \\

OTLRM\,\cite{wang2025otlrm}
& \shortstack[c]{22.85 / 0.699\\0.06662}
& \shortstack[c]{29.54 / 0.847\\0.00808}
& \shortstack[c]{20.97 / 0.612\\0.08274}
& \shortstack[c]{24.69 / 0.685\\0.04926}
& \shortstack[c]{25.16 / 0.736\\0.04853}
& \shortstack[c]{25.97 / 0.795\\0.04510} \\

TNN\,\cite{lu2018tnn}
& \shortstack[c]{24.20 / 0.742\\0.05048}
& \shortstack[c]{28.22 / 0.798\\0.00763}
& \shortstack[c]{21.04 / 0.599\\0.06610}
& \shortstack[c]{22.79 / 0.570\\0.05166}
& \shortstack[c]{24.21 / 0.688\\0.04687}
& \shortstack[c]{25.39 / 0.776\\0.04265} \\

SALTS\,\cite{wu2024salts}
& \shortstack[c]{\underline{26.04} / \underline{0.828}\\\underline{0.04534}}
& \shortstack[c]{29.20 / 0.851\\\underline{0.00535}}
& \shortstack[c]{24.05 / 0.759\\0.05480}
& \shortstack[c]{25.70 / 0.745\\0.03740}
& \shortstack[c]{27.25 / 0.835\\0.03330}
& \shortstack[c]{28.44 / 0.889\\\underline{0.02998}} \\

S$^2$NTNN\,\cite{luo2022s2ntnn}
& \shortstack[c]{25.17 / 0.819\\0.04730}
& \shortstack[c]{\underline{29.57} / \underline{0.890}\\0.00766}
& \shortstack[c]{24.14 / \underline{0.790}\\\underline{0.05318}}
& \shortstack[c]{\underline{26.34} / \underline{0.805}\\\underline{0.03592}}
& \shortstack[c]{27.29 / \underline{0.859}\\\underline{0.03270}}
& \shortstack[c]{28.01 / \underline{0.896}\\0.03028} \\

\textbf{Ours}
& \shortstack[c]{\textbf{27.75} / \textbf{0.908}\\\textbf{0.03147}}
& \shortstack[c]{\textbf{31.00} / \textbf{0.906}\\\textbf{0.00476}}
& \shortstack[c]{\textbf{25.47} / \textbf{0.849}\\\textbf{0.04516}}
& \shortstack[c]{\textbf{27.07} / \textbf{0.838}\\\textbf{0.03079}}
& \shortstack[c]{\textbf{28.12} / \textbf{0.889}\\\textbf{0.02732}}
& \shortstack[c]{\textbf{28.99} / \textbf{0.922}\\\textbf{0.02495}} \\

\bottomrule
\end{tabular*}
\endgroup
\end{table*}

\begin{figure*}[!t]
\centering
\subfloat[Contiguous block mask.]{\includegraphics[width=1\textwidth]{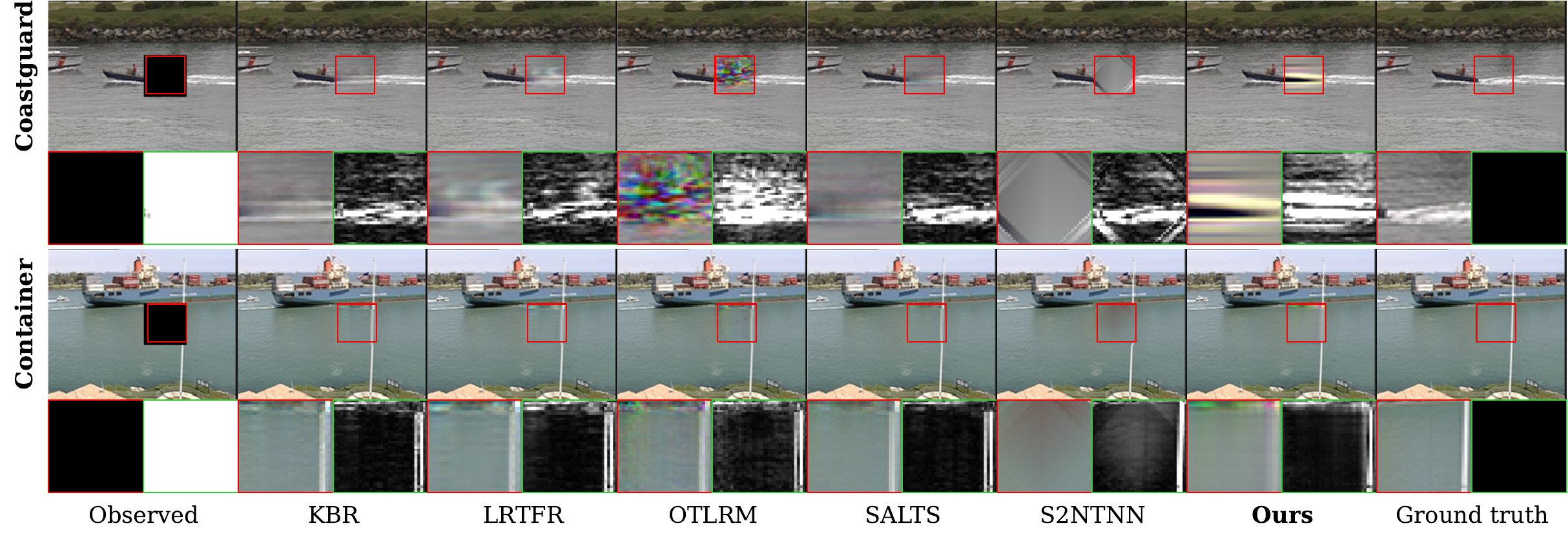}}\\[-0.2mm]
\subfloat[Distributed grid mask.]{\includegraphics[width=1\textwidth]{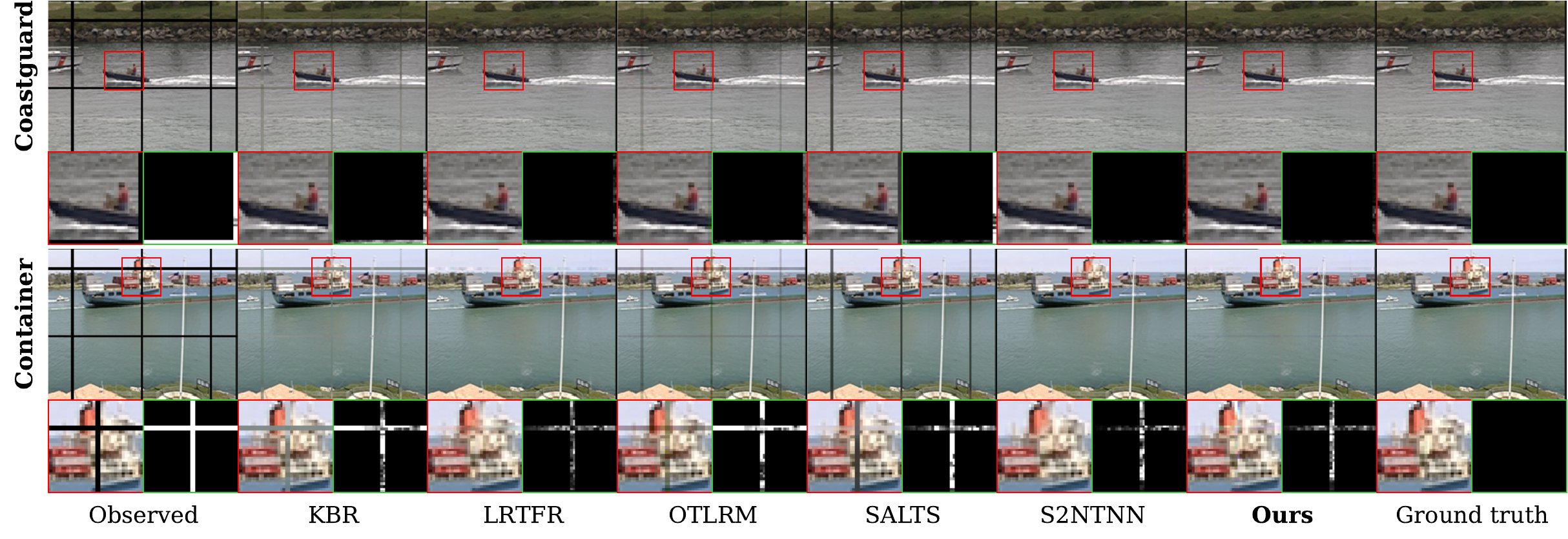}}
\caption{Structured-mask RGB visual comparisons on Coastguard and Container: (a) contiguous block and (b) distributed grid missingness. The observed input, reconstructions, and ground truth are shown with common regions of interest for each video and mask.}
\label{fig:structured-mask-visuals}
\end{figure*}

Figure~\ref{fig:davis-rgb-visual} compares Bear, Drift-chicane, and Cows
at $\mathrm{SR}=0.3$, using a common frame and two fixed ROIs per
sequence. The eight columns contain the observed input, five representative
baselines, Ours, and ground truth. Ours exceeds the strongest baseline in
missing-region PSNR by $1.71$, $1.43$, and $1.18$ dB on Bear,
Drift-chicane, and Cows, respectively
(Table~\ref{tab:davis-rgb-main}(a)). The enlarged regions show differences in edge and
texture recovery, while the displayed PSNR values are computed over the missing
tubes across all 50 frames. Visual comparisons for the remaining five selected sequences are provided in the supplementary material.

\subsection{Structured-mask experiments on color videos}
\label{sec:color-structured}

\begin{table}[!htbp]
\centering
\caption{Structured-mask RGB completion on Coastguard and Container
(about $6\%$ missing tubes). Each result reports missing-region PSNR
(dB) / full-frame SSIM. Best and second-best values for each video
and mask are bold and underlined, respectively.}
\label{tab:color-structured}
\begingroup
\fontsize{7.85}{9.1}\selectfont
\setlength{\tabcolsep}{2pt}
\renewcommand{\arraystretch}{1.1}
\begin{tabular*}{\columnwidth}{@{\extracolsep{\fill}}llcc@{}}
\toprule
Method & Mask & Coastguard & Container \\
\midrule
\multirow{2}{*}{DRO-TFF\,\cite{li2025drotff}} & Block & 15.52 / 0.9447 & 19.57 / 0.9456 \\
 & Grid & 20.35 / 0.9408 & 16.69 / 0.9028 \\
\addlinespace[1pt]
\multirow{2}{*}{HLRTF\,\cite{luo2022hlrtf}} & Block & 15.97 / 0.9473 & 20.81 / 0.9639 \\
 & Grid & 7.46 / 0.8076 & 4.84 / 0.7848 \\
\addlinespace[1pt]
\multirow{2}{*}{KBR\,\cite{xie2018kbr}} & Block & \textbf{19.74} / \textbf{0.9658} & \underline{23.32} / 0.9762 \\
 & Grid & 16.54 / 0.8941 & 14.89 / 0.9216 \\
\addlinespace[1pt]
\multirow{2}{*}{LRTFR\,\cite{luo2024lrtfr}} & Block & \underline{19.39} / \underline{0.9636} & 22.04 / 0.9721 \\
 & Grid & 26.60 / 0.9781 & 25.61 / 0.9809 \\
\addlinespace[1pt]
\multirow{2}{*}{OTLRM\,\cite{wang2025otlrm}} & Block & 13.56 / 0.9424 & 22.09 / 0.9682 \\
 & Grid & 22.15 / 0.9471 & 15.05 / 0.8809 \\
\addlinespace[1pt]
\multirow{2}{*}{TNN\,\cite{lu2018tnn}} & Block & 17.90 / 0.9602 & 22.70 / 0.9753 \\
 & Grid & 6.41 / 0.7993 & 3.65 / 0.7792 \\
\addlinespace[1pt]
\multirow{2}{*}{SALTS\,\cite{wu2024salts}} & Block & 18.45 / 0.9634 & \textbf{23.75} / \textbf{0.9797} \\
 & Grid & 13.10 / 0.8949 & 11.74 / 0.8887 \\
\addlinespace[1pt]
\multirow{2}{*}{S$^2$NTNN\,\cite{luo2022s2ntnn}} & Block & 17.47 / 0.9506 & 20.20 / 0.9735 \\
 & Grid & \underline{29.45} / \underline{0.9862} & \underline{26.11} / \underline{0.9827} \\
\addlinespace[1pt]
\multirow{2}{*}{\textbf{Ours}} & Block & 13.14 / 0.9509 & 22.56 / \underline{0.9763} \\
 & Grid & \textbf{30.15} / \textbf{0.9880} & \textbf{26.40} / \textbf{0.9858} \\
\bottomrule
\end{tabular*}
\endgroup
\end{table}

Table~\ref{tab:color-structured} compares contiguous-block
and distributed-grid masks on Coastguard and Container at
comparable missing fractions. Under grid missingness, Ours
achieves the highest PSNR and SSIM on both videos:
$30.15$ dB/$0.9880$ on Coastguard and $26.40$ dB/$0.9858$
on Container. Its PSNR margins over the second-best methods
are $0.70$ and $0.29$ dB, respectively. For contiguous
blocks, KBR leads on Coastguard and SALTS leads on
Container. Figure~\ref{fig:structured-mask-visuals}
shows the recovered regions under both masks. The grid
pattern leaves observed content around distributed missing
locations, whereas the block pattern removes a contiguous
region; the method rankings differ between these two
sampling geometries.

\FloatBarrier
\section{Conclusion}\label{sec:conclusion}

This paper presented a video completion framework with two
complementary stages. First, a regularized masked KPD
estimator fits a compact structural model to observed
entries using cyclic component least squares and
regularizer-optimal scale balancing. The exact updates
yield nonincreasing objective values. Second, tube-wise
cross-fitting supplies held-out structural prediction
errors for self-supervised residual learning. A lightweight
3-D CNN refines the backbone reconstruction, and exact
data consistency preserves the observations.

Ablation studies show that cross-fitted supervision improves
mean reconstruction quality compared with direct
residual learning. The complete method achieves the highest
mean PSNR and SSIM at the higher grayscale sampling rates
and at every evaluated rate on the eight selected DAVIS RGB
sequences. Method rankings differ between grid and contiguous-block
missingness in the structured-mask experiments. Future work will investigate richer temporal
motion modeling and reconstruction of large contiguous
missing regions.

\end{document}


\maketitle
\noindent This supplement presents the 360p RGB parameter study,
per-sequence results, temporal-gradient MAE, runtime measurements, and
additional visual comparisons. The evaluation includes eight selected
DAVIS sequences, three sampling rates, and nine methods. The configuration
is selected using Blackswan, which is excluded from the eight-sequence
averages. The evaluation tables report measurements for these eight videos.

\appendices

\section{Single-Sequence Hyperparameter Selection for 360p RGB Video Completion}
\label{app:rgb360-hyperparameters}

\subsection{Selection Protocol}
\label{app:rgb360-protocol}

We select the 360p RGB configuration using Blackswan from DAVIS 2017.
The study varies KPD rank, spatial convolutional kernel size, CNN width,
the number of optimization steps, and training-patch size. The sequence contains $50$ RGB frames
resized to $360\times640$ pixels. All configurations use tube-wise sampling
at $\mathrm{SR}=0.3$ with seed $20261014$ and the same mask across frames
and channels. Each channel is reconstructed separately before the RGB
video is assembled.

We select the configuration using missing-region PSNR on Blackswan and
also report full-frame SSIM. Reference values at missing pixels
are used for this development comparison, whereas the cross-fitted CNN is trained
on observed measurements. Settings not varied in a given sweep remain fixed,
including five cross-fitting folds, 12 masked KPD sweeps, batch size two,
and the regularization and learning settings of the implementation.

This configuration is then used for all eight selected RGB videos
at all three sampling rates. Tables in this appendix report the Blackswan
parameter study; the main paper reports the evaluation results.

\subsection{KPD Rank and Spatial Convolutional Kernel}
\label{app:rgb360-rank-kernel}

We vary KPD rank $K\in\{4,8,12\}$ and spatial kernel width
$k_s\in\{3,5\}$, fixing temporal kernel width $k_t=3$, CNN width
$h=12$, training steps at $1200$, and patch size at
$16\times48\times48$. With the smaller spatial kernel, PSNR increases
with rank (Table~\ref{tab:supp-rgb360-rank-kernel}). The
$3\times3\times3$ kernel gives the highest PSNR at $K=12$ and is retained.

\begin{table}[!t]
\centering
\caption{Rank and spatial-kernel sensitivity on Blackswan
($360\times640\times50$, RGB; $\mathrm{SR}=0.3$). The hidden width and
training budget are fixed at $h=12$ and $1200$ steps.}
\label{tab:supp-rgb360-rank-kernel}
\footnotesize
\setlength{\tabcolsep}{5pt}
\begin{tabular}{cccc}
\toprule
Rank $K$ & Conv3D kernel $(t,h,w)$ & Missing PSNR $\uparrow$ & Full SSIM $\uparrow$ \\
\midrule
$4$  & $3\times3\times3$ & 27.1250 & 0.8510 \\
$4$  & $3\times5\times5$ & 27.1250 & 0.8500 \\
$8$  & $3\times3\times3$ & 27.3964 & 0.8589 \\
$8$  & $3\times5\times5$ & 26.7135 & 0.8393 \\
$12$ & $3\times3\times3$ & \textbf{27.6069} & \textbf{0.8638} \\
$12$ & $3\times5\times5$ & 26.8058 & 0.8408 \\
\bottomrule
\end{tabular}
\end{table}

\begin{figure}[!t]
\centering
\includegraphics[width=0.88\columnwidth]{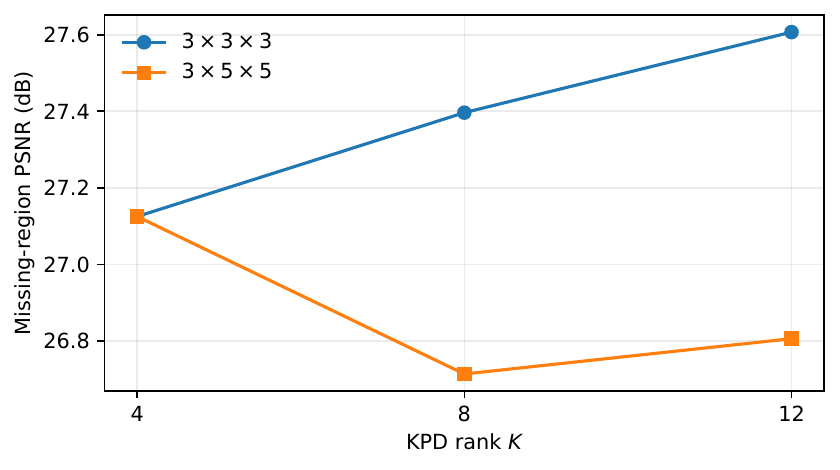}
\caption{Missing-region PSNR versus KPD rank for two spatial convolutional
kernel sizes on the designated development sequence. All other settings are
fixed as in Table~\ref{tab:supp-rgb360-rank-kernel}.}
\label{fig:supp-rgb360-rank-kernel}
\end{figure}

\subsection{Structural Rank Versus CNN Width}
\label{app:rgb360-rank-width}

We compare the effects of increasing backbone rank and residual-network
width. The kernel size is fixed at $3\times3\times3$, with $1200$ training steps and
$16\times48\times48$ patches. The rank sweep uses $h=12$, and the width
sweep fixes $K=12$; both include the reference setting $(K,h)=(12,12)$.

Table~\ref{tab:supp-rgb360-rank-width} shows that increasing rank above
$K=12$ changes PSNR only modestly while increasing runtime. At $K=12$,
wider CNNs improve PSNR and SSIM, with the highest values at $h=32$.
We use $K=12$ and $h=32$ in the training-budget study.

\begin{table}[!t]
\centering
\caption{KPD-rank and CNN-width sensitivity on Blackswan at
$\mathrm{SR}=0.3$, with a $3\times3\times3$ kernel and $1200$ steps.
Reported runtimes are measured end to end for each video (seconds).}
\label{tab:supp-rgb360-rank-width}
\footnotesize
\setlength{\tabcolsep}{5pt}
\begin{tabular}{ccccc}
\toprule
Rank $K$ & Width $h$ & Missing PSNR $\uparrow$ & Full SSIM $\uparrow$ & Time (s) \\
\midrule
$12$ & $12$ & 27.6069 & 0.8638 & 832.7 \\
$12$ & $24$ & 27.7573 & 0.8722 & 833.5 \\
$12$ & $32$ & \textbf{27.8172} & \textbf{0.8734} & 833.3 \\
\midrule
$16$ & $12$ & 27.5599 & 0.8616 & 1088.5 \\
$20$ & $12$ & 27.6258 & 0.8626 & 1344.0 \\
$24$ & $12$ & 27.5880 & 0.8610 & 1594.7 \\
\bottomrule
\end{tabular}
\end{table}

\subsection{Optimization Steps and Training-Patch Size}
\label{app:rgb360-steps-patch}

With $K=12$, $h=32$, and a $3\times3\times3$ kernel fixed, we vary
the number of training steps and spatial patch size while keeping the temporal patch length
at 16 frames. Table~\ref{tab:supp-rgb360-steps-patch} reports increasing
PSNR with larger patches and longer training in the tested ranges, with
nonmonotone behavior at the largest budget. The highest measured PSNR
is obtained with $3600$ steps and $96\times96$ spatial patches. Changing patch size alters
the training context but leaves the CNN receptive field unchanged.

\begin{table}[!t]
\centering
\caption{Training-budget and patch-size study on Blackswan at
$\mathrm{SR}=0.3$; $K=12$, $h=32$, and a $3\times3\times3$ kernel are fixed.
All patches have temporal length $16$.}
\label{tab:supp-rgb360-steps-patch}
\footnotesize
\setlength{\tabcolsep}{6pt}
\begin{tabular}{ccccc}
\toprule
Setting & Steps & Spatial patch & Missing PSNR $\uparrow$ & Full SSIM $\uparrow$ \\
\midrule
A0  & 1200 & $48\times48$ & 27.8172 & 0.8734 \\
S2  & 2400 & $48\times48$ & 28.2571 & 0.8884 \\
P64 & 2400 & $64\times64$ & 28.6947 & 0.8976 \\
P96 & 2400 & $96\times96$ & 28.7211 & 0.8997 \\
S3  & 3600 & $48\times48$ & 28.8627 & 0.9027 \\
C1  & 3600 & $64\times64$ & 28.8469 & 0.9018 \\
\textbf{C2}  & \textbf{3600} & $\mathbf{96\times96}$ & \textbf{28.9023} & \textbf{0.9047} \\
C3  & 4800 & $48\times48$ & 28.6160 & 0.8957 \\
\bottomrule
\end{tabular}
\end{table}

\begin{figure}[!t]
\centering
\includegraphics[width=0.88\columnwidth]{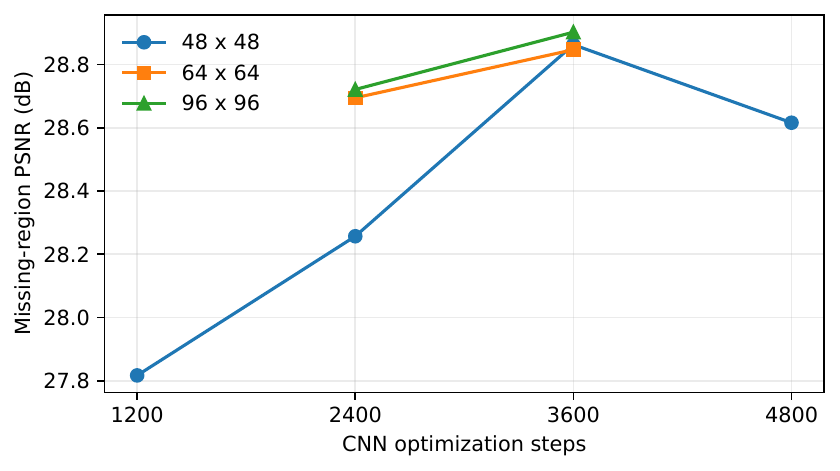}
\caption{Sensitivity to the optimization budget at three training-patch
sizes on the development sequence. Curves connect only the evaluated configurations.}
\label{fig:supp-rgb360-steps-patch}
\end{figure}

\subsection{Final RGB Configuration}
\label{app:rgb360-frozen}

The development study selects the following fixed 360p configuration:
\begin{equation}
\boxed{\begin{gathered}
K=12,\quad h=32,\quad (k_t,k_s,k_s)=(3,3,3),\\
N_{\mathrm{step}}=3600,\quad (T_p,H_p,W_p)=(16,96,96).
\end{gathered}}
\label{eq:supp-rgb360-frozen}
\end{equation}
Five cross-fitting folds, 12 masked KPD sweeps, and the remaining
regularization settings are unchanged. Relative to the initial $K=4$,
$h=12$, $1200$-step configuration with $16\times48\times48$ patches,
the selected configuration improves missing-region PSNR on Blackswan from
$27.1250$ to $28.9023$~dB and full-frame SSIM from $0.8510$ to $0.9047$.
These values are measured on the development sequence used for selection.

We apply this configuration to all eight selected RGB sequences at all
three sampling rates. Each sequence is reconstructed using its own
self-supervised fit with the same hyperparameter settings.


\section{Supplementary Evaluation Results}
\label{app:supp-evaluation}

\subsection{Selected-Sequence Quantitative Comparison}
\label{app:selected-davis}

The eight selected DAVIS 2017 videos are Bear, Drift-chicane, Cows, Dog,
Camel, Kite-surf, Parkour, and Breakdance. Each contains $50$ RGB frames
resized to $360\times640$ pixels. For each video and sampling rate,
all methods use the same tube-wise mask with seed $20261014$. The
configuration selected on Blackswan (Appendix~\ref{app:rgb360-hyperparameters})
is applied to the eight sequences. Tables~\ref{tab:supp-davis-psnr-sr20}--
\ref{tab:supp-davis-psnr-sr40} report missing-region PSNR, and
Tables~\ref{tab:supp-davis-ssim-sr20}--\ref{tab:supp-davis-ssim-sr40}
report full-frame SSIM. The mean rows give the eight-video
averages reported in Table~V of the main paper. Table~\ref{tab:supp-rgb-selected-runtime}
reports runtimes for the same implementations.

\begin{table}[!t]
\centering
\caption{Median wall-clock time (seconds) per reconstructed RGB video in the comparison of eight selected sequences. Times are measured using the original CPU/GPU implementations.}
\label{tab:supp-rgb-selected-runtime}
\footnotesize
\setlength{\tabcolsep}{4pt}
\begin{tabular}{lc}
\toprule
Method & Median time (s) \\
\midrule
DRO-TFF & 98.2 \\
HLRTF & 62.6 \\
KBR & 2851.5 \\
LRTFR & 45.4 \\
OTLRM & 434.1 \\
TNN & 926.1 \\
SALTS & 16092.2 \\
S$^2$NTNN & 466.9 \\
\textbf{Ours} & 1859.2 \\
\bottomrule
\end{tabular}
\end{table}


\begin{table*}[!t]
\centering
\caption{Missing-region PSNR (dB) on the eight selected DAVIS RGB sequences at $\mathrm{SR}=0.2$. Best and second-best values per sequence are bold and underlined.}
\label{tab:supp-davis-psnr-sr20}
\scriptsize
\setlength{\tabcolsep}{3pt}
\resizebox{\textwidth}{!}{%
\begin{tabular}{lccccccccc}
\toprule
Sequence & DRO-TFF & HLRTF & KBR & LRTFR & OTLRM & TNN & SALTS & S$^2$NTNN & \textbf{Ours}\\
\midrule
bear & 18.49 & 23.21 & 24.59 & 23.40 & 22.46 & 22.72 & \underline{24.88} & 24.31 & \textbf{26.85} \\
drift-chicane & 21.87 & 28.12 & 23.38 & 27.24 & \underline{28.95} & 26.86 & 27.33 & 28.78 & \textbf{29.89} \\
cows & 16.16 & 20.98 & 22.75 & 22.61 & 20.79 & 19.85 & 22.80 & \underline{23.21} & \textbf{24.47} \\
dog & 24.78 & 25.32 & 29.07 & 24.47 & 26.83 & 23.83 & 28.55 & \underline{29.43} & \textbf{29.97} \\
camel & 17.06 & 21.85 & 22.39 & 22.68 & 22.45 & 20.75 & 23.93 & \underline{24.15} & \textbf{24.50} \\
kite-surf & 25.26 & 27.22 & \underline{30.52} & 26.92 & 28.83 & 26.33 & 30.33 & 29.71 & \textbf{30.88} \\
parkour & 23.50 & 23.99 & \underline{27.09} & 23.06 & 25.23 & 23.05 & 26.78 & \textbf{27.36} & 26.87 \\
breakdance & 18.06 & 20.92 & 20.69 & 22.12 & 21.96 & 18.91 & 20.99 & \textbf{23.77} & \underline{23.12} \\
\midrule
Mean & 20.65 & 23.95 & 25.06 & 24.06 & 24.69 & 22.79 & 25.70 & \underline{26.34} & \textbf{27.07} \\
\bottomrule
\end{tabular}}
\end{table*}


\begin{table*}[!t]
\centering
\caption{Missing-region PSNR (dB) on the eight selected DAVIS RGB sequences at $\mathrm{SR}=0.3$. Best and second-best values per sequence are bold and underlined.}
\label{tab:supp-davis-psnr-sr30}
\scriptsize
\setlength{\tabcolsep}{3pt}
\resizebox{\textwidth}{!}{%
\begin{tabular}{lccccccccc}
\toprule
Sequence & DRO-TFF & HLRTF & KBR & LRTFR & OTLRM & TNN & SALTS & S$^2$NTNN & \textbf{Ours}\\
\midrule
bear & 22.96 & 24.11 & 25.99 & 24.01 & 22.85 & 24.20 & \underline{26.04} & 25.17 & \textbf{27.75} \\
drift-chicane & 23.36 & 29.51 & 27.95 & 28.31 & 29.54 & 28.22 & 29.20 & \underline{29.57} & \textbf{31.00} \\
cows & 19.78 & 21.86 & \underline{24.29} & 23.40 & 20.97 & 21.04 & 24.05 & 24.14 & \textbf{25.47} \\
dog & 25.50 & 26.92 & \underline{30.74} & 25.25 & 27.39 & 25.32 & 30.10 & 30.63 & \textbf{31.18} \\
camel & 21.66 & 23.02 & \underline{25.15} & 23.25 & 21.76 & 22.17 & 25.07 & 25.12 & \textbf{25.46} \\
kite-surf & 27.64 & 28.23 & \underline{31.99} & 27.42 & 29.67 & 27.66 & 31.71 & 30.70 & \textbf{32.01} \\
parkour & 23.91 & 25.53 & \textbf{28.50} & 23.75 & 25.95 & 24.47 & 28.02 & \underline{28.16} & 27.95 \\
breakdance & 21.20 & 22.91 & 24.13 & 23.75 & 23.15 & 20.64 & 23.80 & \textbf{24.85} & \underline{24.13} \\
\midrule
Mean & 23.25 & 25.26 & \underline{27.34} & 24.89 & 25.16 & 24.21 & 27.25 & 27.29 & \textbf{28.12} \\
\bottomrule
\end{tabular}}
\end{table*}


\begin{table*}[!t]
\centering
\caption{Missing-region PSNR (dB) on the eight selected DAVIS RGB sequences at $\mathrm{SR}=0.4$. Best and second-best values per sequence are bold and underlined.}
\label{tab:supp-davis-psnr-sr40}
\scriptsize
\setlength{\tabcolsep}{3pt}
\resizebox{\textwidth}{!}{%
\begin{tabular}{lccccccccc}
\toprule
Sequence & DRO-TFF & HLRTF & KBR & LRTFR & OTLRM & TNN & SALTS & S$^2$NTNN & \textbf{Ours}\\
\midrule
bear & 23.51 & 24.72 & \underline{27.24} & 24.50 & 23.88 & 25.34 & 26.95 & 25.84 & \textbf{28.57} \\
drift-chicane & 24.54 & 29.55 & 30.39 & 29.43 & \underline{30.69} & 29.29 & 30.69 & 29.95 & \textbf{32.28} \\
cows & 21.98 & 22.31 & \underline{25.44} & 23.93 & 21.31 & 22.04 & 25.05 & 24.86 & \textbf{26.04} \\
dog & 25.60 & 27.85 & \textbf{32.38} & 25.92 & 28.17 & 26.64 & 31.46 & 31.51 & \underline{32.30} \\
camel & 22.26 & 23.69 & \textbf{26.28} & 23.60 & 22.75 & 23.27 & 25.98 & 25.82 & \underline{26.04} \\
kite-surf & 27.91 & 29.14 & \textbf{33.38} & 27.79 & 30.06 & 28.78 & 32.90 & 31.38 & \underline{32.98} \\
parkour & 24.05 & 26.40 & \textbf{29.75} & 24.45 & 26.81 & 25.64 & 29.04 & \underline{29.11} & 28.79 \\
breakdance & 22.50 & 23.90 & \textbf{25.66} & 24.99 & 24.07 & 22.08 & 25.49 & \underline{25.59} & 24.90 \\
\midrule
Mean & 24.04 & 25.94 & \underline{28.81} & 25.58 & 25.97 & 25.39 & 28.44 & 28.01 & \textbf{28.99} \\
\bottomrule
\end{tabular}}
\end{table*}


\begin{table*}[!t]
\centering
\caption{Full-frame SSIM on the eight selected DAVIS RGB sequences at $\mathrm{SR}=0.2$. Best and second-best values per sequence are bold and underlined.}
\label{tab:supp-davis-ssim-sr20}
\scriptsize
\setlength{\tabcolsep}{3pt}
\resizebox{\textwidth}{!}{%
\begin{tabular}{lccccccccc}
\toprule
Sequence & DRO-TFF & HLRTF & KBR & LRTFR & OTLRM & TNN & SALTS & S$^2$NTNN & \textbf{Ours}\\
\midrule
bear & 0.320 & 0.642 & 0.715 & 0.649 & 0.639 & 0.622 & 0.744 & \underline{0.749} & \textbf{0.868} \\
drift-chicane & 0.316 & 0.780 & 0.502 & 0.707 & 0.811 & 0.717 & 0.755 & \underline{0.854} & \textbf{0.861} \\
cows & 0.271 & 0.537 & 0.647 & 0.627 & 0.554 & 0.464 & 0.649 & \underline{0.710} & \textbf{0.785} \\
dog & 0.588 & 0.632 & 0.820 & 0.593 & 0.724 & 0.536 & 0.802 & \underline{0.848} & \textbf{0.872} \\
camel & 0.264 & 0.572 & 0.671 & 0.648 & 0.637 & 0.517 & 0.726 & \underline{0.770} & \textbf{0.801} \\
kite-surf & 0.591 & 0.742 & 0.852 & 0.741 & 0.793 & 0.677 & \underline{0.865} & 0.850 & \textbf{0.903} \\
parkour & 0.569 & 0.600 & 0.738 & 0.656 & 0.656 & 0.546 & 0.768 & \textbf{0.840} & \underline{0.831} \\
breakdance & 0.430 & 0.611 & 0.579 & 0.691 & 0.665 & 0.485 & 0.650 & \textbf{0.820} & \underline{0.788} \\
\midrule
Mean & 0.419 & 0.640 & 0.691 & 0.664 & 0.685 & 0.570 & 0.745 & \underline{0.805} & \textbf{0.838} \\
\bottomrule
\end{tabular}}
\end{table*}


\begin{table*}[!t]
\centering
\caption{Full-frame SSIM on the eight selected DAVIS RGB sequences at $\mathrm{SR}=0.3$. Best and second-best values per sequence are bold and underlined.}
\label{tab:supp-davis-ssim-sr30}
\scriptsize
\setlength{\tabcolsep}{3pt}
\resizebox{\textwidth}{!}{%
\begin{tabular}{lccccccccc}
\toprule
Sequence & DRO-TFF & HLRTF & KBR & LRTFR & OTLRM & TNN & SALTS & S$^2$NTNN & \textbf{Ours}\\
\midrule
bear & 0.652 & 0.733 & 0.811 & 0.732 & 0.699 & 0.742 & \underline{0.828} & 0.819 & \textbf{0.908} \\
drift-chicane & 0.428 & 0.846 & 0.750 & 0.792 & 0.847 & 0.798 & 0.851 & \underline{0.890} & \textbf{0.906} \\
cows & 0.527 & 0.640 & 0.765 & 0.724 & 0.612 & 0.599 & 0.759 & \underline{0.790} & \textbf{0.849} \\
dog & 0.674 & 0.756 & 0.891 & 0.679 & 0.784 & 0.669 & 0.877 & \underline{0.897} & \textbf{0.916} \\
camel & 0.600 & 0.683 & 0.787 & 0.727 & 0.649 & 0.646 & 0.808 & \underline{0.834} & \textbf{0.858} \\
kite-surf & 0.761 & 0.807 & 0.907 & 0.790 & 0.841 & 0.767 & \underline{0.913} & 0.890 & \textbf{0.937} \\
parkour & 0.640 & 0.709 & 0.822 & 0.716 & 0.707 & 0.656 & 0.835 & \underline{0.876} & \textbf{0.882} \\
breakdance & 0.640 & 0.748 & 0.778 & 0.796 & 0.747 & 0.626 & 0.809 & \textbf{0.875} & \underline{0.856} \\
\midrule
Mean & 0.615 & 0.740 & 0.814 & 0.744 & 0.736 & 0.688 & 0.835 & \underline{0.859} & \textbf{0.889} \\
\bottomrule
\end{tabular}}
\end{table*}


\begin{table*}[!t]
\centering
\caption{Full-frame SSIM on the eight selected DAVIS RGB sequences at $\mathrm{SR}=0.4$. Best and second-best values per sequence are bold and underlined.}
\label{tab:supp-davis-ssim-sr40}
\scriptsize
\setlength{\tabcolsep}{3pt}
\resizebox{\textwidth}{!}{%
\begin{tabular}{lccccccccc}
\toprule
Sequence & DRO-TFF & HLRTF & KBR & LRTFR & OTLRM & TNN & SALTS & S$^2$NTNN & \textbf{Ours}\\
\midrule
bear & 0.731 & 0.789 & 0.878 & 0.796 & 0.778 & 0.821 & \underline{0.880} & 0.868 & \textbf{0.937} \\
drift-chicane & 0.536 & 0.834 & 0.857 & 0.860 & 0.892 & 0.855 & 0.904 & \underline{0.908} & \textbf{0.938} \\
cows & 0.693 & 0.705 & 0.842 & 0.791 & 0.678 & 0.705 & 0.833 & \underline{0.846} & \textbf{0.887} \\
dog & 0.726 & 0.827 & \underline{0.936} & 0.753 & 0.832 & 0.772 & 0.923 & 0.928 & \textbf{0.945} \\
camel & 0.687 & 0.747 & 0.857 & 0.785 & 0.734 & 0.743 & 0.866 & \underline{0.878} & \textbf{0.895} \\
kite-surf & 0.803 & 0.856 & \underline{0.944} & 0.830 & 0.867 & 0.835 & 0.944 & 0.918 & \textbf{0.959} \\
parkour & 0.689 & 0.770 & 0.882 & 0.774 & 0.773 & 0.744 & 0.882 & \underline{0.913} & \textbf{0.915} \\
breakdance & 0.744 & 0.799 & 0.858 & 0.861 & 0.806 & 0.735 & 0.885 & \textbf{0.909} & \underline{0.901} \\
\midrule
Mean & 0.701 & 0.791 & 0.882 & 0.806 & 0.795 & 0.776 & 0.889 & \underline{0.896} & \textbf{0.922} \\
\bottomrule
\end{tabular}}
\end{table*}


\subsection{Temporal-gradient MAE on Missing Tubes}
\label{app:rgb-temporal-mae}

Temporal-gradient MAE (T-MAE) measures the reconstruction error in
frame-to-frame intensity changes at missing tubes. Under tube-wise sampling,
the set $\mathcal M$ of missing pixel--channel indices remains fixed across
frames. For the reference video $X$ and the data-consistent reconstruction
$\widehat X$, we compute
\begin{equation}
\label{eq:supp-temporal-mae}
\begin{aligned}
\mathrm{TMAE}_{\mathrm{miss}}
&=\frac{1}{(T-1)|\mathcal M|}
\sum_{t=1}^{T-1}\sum_{p\in\mathcal M}
\left|\Delta_t\widehat X_p-\Delta_t X_p\right|,\\
\Delta_t X_p&=X_{p,t+1}-X_{p,t},\quad p=(i,j,c).
\end{aligned}
\end{equation}
Video intensities are normalized to $[0,1]$, and lower T-MAE indicates
more accurate temporal differences at missing locations. The main paper uses
the same metric. Tables~\ref{tab:supp-davis-temporal-sr20}--
\ref{tab:supp-davis-temporal-sr40} report results for all nine methods at each rate
on the eight selected videos. Ours has the lowest T-MAE in $19$ of the
$24$ video--rate cases; the mean results appear in
Table~\ref{tab:supp-davis-temporal-mean}.

\begin{table*}[!t]
\centering
\caption{Temporal-gradient MAE on the missing tubes of eight selected DAVIS 360p RGB videos (lower is better). The rightmost column is the mean across the three sampling rates. The lowest and second-lowest values are bold and underlined, respectively.}
\label{tab:supp-davis-temporal-mean}
\begingroup
\renewcommand{\arraystretch}{1.06}
\setlength{\tabcolsep}{3pt}
\footnotesize
\begin{tabular*}{0.72\textwidth}{@{\extracolsep{\fill}}lrrrr@{}}
\toprule
Method & SR=0.2 & SR=0.3 & SR=0.4 & Mean \\
\midrule
DRO-TFF & 0.05233 & 0.04917 & 0.04842 & 0.04997 \\
HLRTF & 0.04856 & 0.04419 & 0.04177 & 0.04484 \\
KBR & 0.04038 & 0.03505 & 0.03024 & 0.03522 \\
LRTFR & 0.04876 & 0.04687 & 0.04511 & 0.04691 \\
OTLRM & 0.04926 & 0.04853 & 0.04510 & 0.04763 \\
TNN & 0.05166 & 0.04687 & 0.04265 & 0.04706 \\
SALTS & 0.03740 & 0.03330 & \underline{0.02998} & 0.03356 \\
S$^2$NTNN & \underline{0.03592} & \underline{0.03270} & 0.03028 & \underline{0.03296} \\
Ours & \textbf{0.03079} & \textbf{0.02732} & \textbf{0.02495} & \textbf{0.02769} \\
\bottomrule
\end{tabular*}
\endgroup
\end{table*}


\begin{table*}[!t]
\centering
\caption{Missing-tube temporal-gradient MAE (lower is better) on the eight selected DAVIS RGB videos at $\mathrm{SR}=0.2$. Lowest and second-lowest entries in each row are bold and underlined.}
\label{tab:supp-davis-temporal-sr20}
\begingroup
\renewcommand{\arraystretch}{1.05}
\setlength{\tabcolsep}{2.4pt}
\footnotesize
\begin{tabular*}{\textwidth}{@{\extracolsep{\fill}}lrrrrrrrrr@{}}
\toprule
Video & DRO-TFF & HLRTF & KBR & LRTFR & OTLRM & TNN & SALTS & S$^2$NTNN & Ours \\
\midrule
bear & 0.05797 & 0.05304 & 0.05298 & 0.05348 & 0.06816 & 0.05450 & \underline{0.04945} & 0.05136 & \textbf{0.03522} \\
drift-chicane & 0.00984 & 0.00788 & 0.00826 & 0.00778 & 0.00774 & 0.00829 & \underline{0.00582} & 0.00772 & \textbf{0.00519} \\
cows & 0.06826 & 0.06613 & 0.06049 & 0.06096 & 0.08213 & 0.06906 & 0.05932 & \underline{0.05791} & \textbf{0.04894} \\
dog & 0.05979 & 0.05876 & 0.03597 & 0.06479 & 0.04745 & 0.06331 & 0.03790 & \underline{0.03373} & \textbf{0.02997} \\
camel & 0.07584 & 0.06873 & 0.06159 & 0.06715 & 0.07146 & 0.07220 & 0.05667 & \underline{0.05437} & \textbf{0.04865} \\
kite-surf & 0.03892 & 0.03644 & 0.02622 & 0.03664 & 0.03264 & 0.03894 & \underline{0.02473} & 0.02761 & \textbf{0.02046} \\
parkour & 0.06584 & 0.06157 & 0.04328 & 0.06902 & 0.05377 & 0.06577 & 0.04247 & \textbf{0.03483} & \underline{0.03624} \\
breakdance & 0.04215 & 0.03595 & 0.03423 & 0.03025 & 0.03071 & 0.04124 & 0.02285 & \textbf{0.01981} & \underline{0.02162} \\
\midrule
\textbf{Mean} & 0.05233 & 0.04856 & 0.04038 & 0.04876 & 0.04926 & 0.05166 & 0.03740 & \underline{0.03592} & \textbf{0.03079} \\
\bottomrule
\end{tabular*}
\endgroup
\end{table*}


\begin{table*}[!t]
\centering
\caption{Missing-tube temporal-gradient MAE (lower is better) on the eight selected DAVIS RGB videos at $\mathrm{SR}=0.3$. Lowest and second-lowest entries in each row are bold and underlined.}
\label{tab:supp-davis-temporal-sr30}
\begingroup
\renewcommand{\arraystretch}{1.05}
\setlength{\tabcolsep}{2.4pt}
\footnotesize
\begin{tabular*}{\textwidth}{@{\extracolsep{\fill}}lrrrrrrrrr@{}}
\toprule
Video & DRO-TFF & HLRTF & KBR & LRTFR & OTLRM & TNN & SALTS & S$^2$NTNN & Ours \\
\midrule
bear & 0.05461 & 0.05164 & 0.04833 & 0.05276 & 0.06662 & 0.05048 & \underline{0.04534} & 0.04730 & \textbf{0.03147} \\
drift-chicane & 0.00883 & 0.00777 & 0.00707 & 0.00766 & 0.00808 & 0.00763 & \underline{0.00535} & 0.00766 & \textbf{0.00476} \\
cows & 0.06434 & 0.06312 & 0.05631 & 0.05933 & 0.08274 & 0.06610 & 0.05480 & \underline{0.05318} & \textbf{0.04516} \\
dog & 0.05721 & 0.05001 & 0.02944 & 0.06180 & 0.04395 & 0.05518 & 0.03148 & \underline{0.02944} & \textbf{0.02509} \\
camel & 0.07078 & 0.06277 & 0.05362 & 0.06542 & 0.07754 & 0.06676 & 0.05120 & \underline{0.04871} & \textbf{0.04335} \\
kite-surf & 0.03616 & 0.03467 & 0.02196 & 0.03593 & 0.02973 & 0.03530 & \underline{0.02115} & 0.02575 & \textbf{0.01739} \\
parkour & 0.06342 & 0.05234 & 0.03674 & 0.06507 & 0.04946 & 0.05763 & 0.03713 & \textbf{0.03157} & \underline{0.03175} \\
breakdance & 0.03799 & 0.03115 & 0.02695 & 0.02699 & 0.03009 & 0.03585 & 0.01998 & \textbf{0.01795} & \underline{0.01962} \\
\midrule
\textbf{Mean} & 0.04917 & 0.04419 & 0.03505 & 0.04687 & 0.04853 & 0.04687 & 0.03330 & \underline{0.03270} & \textbf{0.02732} \\
\bottomrule
\end{tabular*}
\endgroup
\end{table*}


\begin{table*}[!t]
\centering
\caption{Missing-tube temporal-gradient MAE (lower is better) on the eight selected DAVIS RGB videos at $\mathrm{SR}=0.4$. Lowest and second-lowest entries in each row are bold and underlined.}
\label{tab:supp-davis-temporal-sr40}
\begingroup
\renewcommand{\arraystretch}{1.05}
\setlength{\tabcolsep}{2.4pt}
\footnotesize
\begin{tabular*}{\textwidth}{@{\extracolsep{\fill}}lrrrrrrrrr@{}}
\toprule
Video & DRO-TFF & HLRTF & KBR & LRTFR & OTLRM & TNN & SALTS & S$^2$NTNN & Ours \\
\midrule
bear & 0.05395 & 0.05057 & \underline{0.04169} & 0.05203 & 0.06059 & 0.04679 & 0.04171 & 0.04420 & \textbf{0.02885} \\
drift-chicane & 0.00843 & 0.00767 & 0.00642 & 0.00755 & 0.00666 & 0.00707 & \underline{0.00501} & 0.00763 & \textbf{0.00443} \\
cows & 0.06266 & 0.06194 & 0.05034 & 0.05808 & 0.08042 & 0.06301 & 0.05055 & \underline{0.04977} & \textbf{0.04244} \\
dog & 0.05663 & 0.04574 & \underline{0.02407} & 0.05843 & 0.04151 & 0.04815 & 0.02676 & 0.02664 & \textbf{0.02211} \\
camel & 0.06962 & 0.05941 & 0.04698 & 0.06421 & 0.06963 & 0.06179 & 0.04642 & \underline{0.04517} & \textbf{0.03980} \\
kite-surf & 0.03591 & 0.03277 & \underline{0.01821} & 0.03525 & 0.02913 & 0.03193 & 0.01832 & 0.02406 & \textbf{0.01528} \\
parkour & 0.06258 & 0.04796 & 0.03161 & 0.06076 & 0.04526 & 0.05107 & 0.03319 & \underline{0.02863} & \textbf{0.02845} \\
breakdance & 0.03762 & 0.02810 & 0.02264 & 0.02461 & 0.02761 & 0.03143 & \underline{0.01788} & \textbf{0.01613} & 0.01823 \\
\midrule
\textbf{Mean} & 0.04842 & 0.04177 & 0.03024 & 0.04511 & 0.04510 & 0.04265 & \underline{0.02998} & 0.03028 & \textbf{0.02495} \\
\bottomrule
\end{tabular*}
\endgroup
\end{table*}


\subsection{Quality Trends and Computational Cost}
\label{app:selected-trends}

Figure~7 of the main paper shows mean PSNR, SSIM, and T-MAE
across sampling rates. Figure~\ref{fig:supp-selected-margins}
shows per-sequence missing-region PSNR margins over
the strongest alternative at each rate. The median
wall-clock time for Ours is $1859.2$ s per RGB video,
compared with $466.9$ s for S$^2$NTNN and $45.4$ s for
LRTFR, under the implementations reported in
Table~\ref{tab:supp-rgb-selected-runtime}.

\begin{figure*}[!t]
\centering
\includegraphics[width=1.0\textwidth]{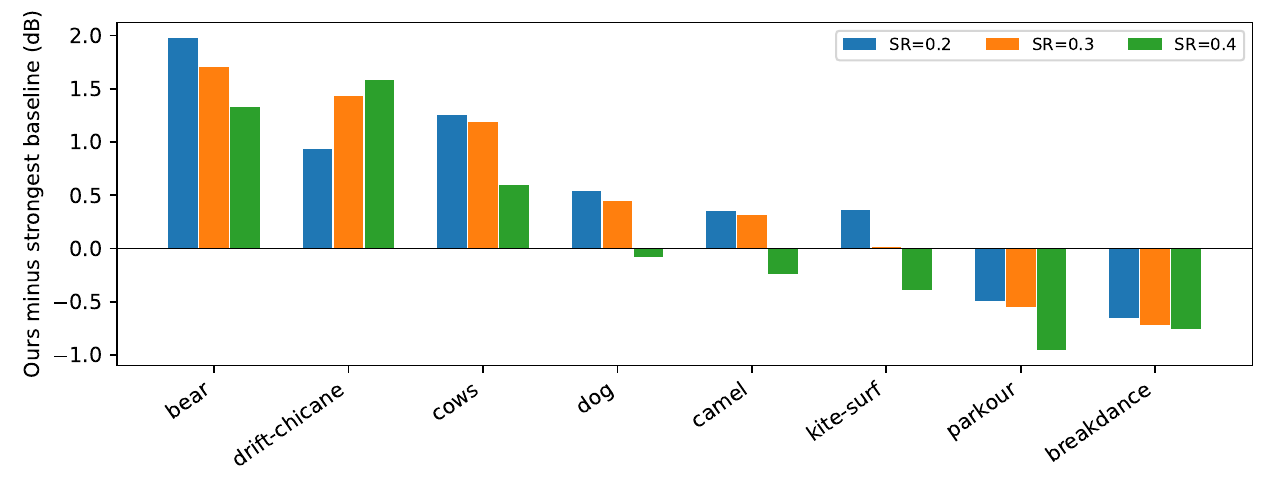}
\caption{Difference in missing-region PSNR between Ours and the
strongest alternative baseline for each selected sequence and sampling rate.}
\label{fig:supp-selected-margins}
\end{figure*}

\subsection{Additional RGB Visual Comparisons}
\label{app:rgb-selected-visuals}

The main paper shows Bear, Drift-chicane, and Cows. Comparisons for the remaining
five selected videos are shown below at sampling rates $0.2$, $0.3$,
and $0.4$ (from top to bottom). For each video, all methods use the same displayed frame and two
fixed ROIs. The eleven columns show the observed input,
nine reconstructions, and ground truth. The PSNR labels report the
missing-region score over the full 50-frame sequence.

\begin{figure*}[!t]
\centering
\includegraphics[page=1,width=1.0\textwidth]{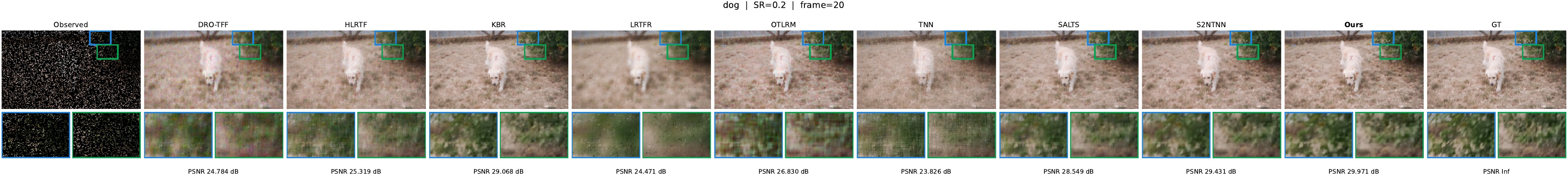}\\[1.0mm]\includegraphics[page=2,width=1.0\textwidth]{figures/supplementary/full_visuals/per_video_supplement_11col_dog.pdf}\\[1.0mm]\includegraphics[page=3,width=1.0\textwidth]{figures/supplementary/full_visuals/per_video_supplement_11col_dog.pdf}
\caption{RGB reconstruction comparison on Dog at sampling rates 0.2, 0.3, and 0.4 (from top to bottom), with common ROIs across methods. The annotated PSNR values report missing-region PSNR over all 50 frames.}
\label{fig:supp-rgb-dog}
\end{figure*}

\begin{figure*}[!t]
\centering
\includegraphics[page=1,width=1.0\textwidth]{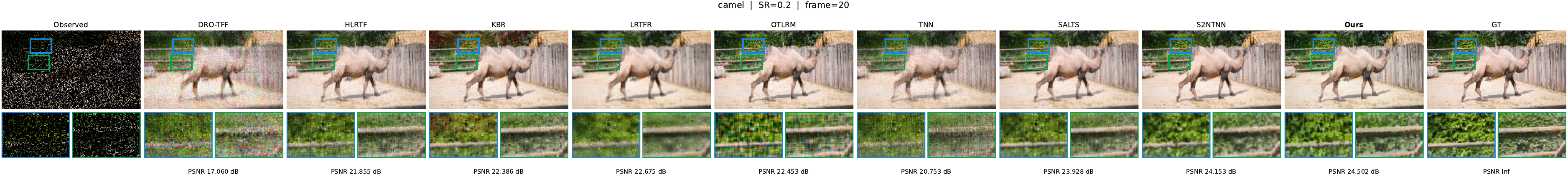}\\[1.0mm]\includegraphics[page=2,width=1.0\textwidth]{figures/supplementary/full_visuals/per_video_supplement_11col_camel.pdf}\\[1.0mm]\includegraphics[page=3,width=1.0\textwidth]{figures/supplementary/full_visuals/per_video_supplement_11col_camel.pdf}
\caption{RGB reconstruction comparison on Camel at sampling rates 0.2, 0.3, and 0.4 (from top to bottom), with common ROIs across methods. The annotated PSNR values report missing-region PSNR over all 50 frames.}
\label{fig:supp-rgb-camel}
\end{figure*}

\begin{figure*}[!t]
\centering
\includegraphics[page=1,width=1.0\textwidth]{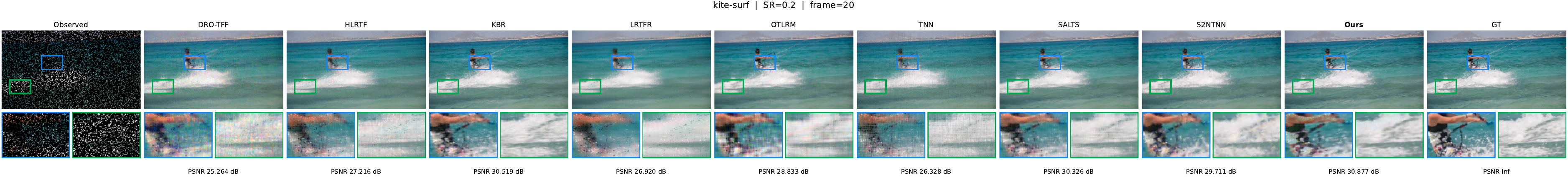}\\[1.0mm]\includegraphics[page=2,width=1.0\textwidth]{figures/supplementary/full_visuals/per_video_supplement_11col_kite-surf.pdf}\\[1.0mm]\includegraphics[page=3,width=1.0\textwidth]{figures/supplementary/full_visuals/per_video_supplement_11col_kite-surf.pdf}
\caption{RGB reconstruction comparison on Kite-surf at sampling rates 0.2, 0.3, and 0.4 (from top to bottom), with common ROIs across methods. The annotated PSNR values report missing-region PSNR over all 50 frames.}
\label{fig:supp-rgb-kite-surf}
\end{figure*}

\begin{figure*}[!t]
\centering
\includegraphics[page=1,width=1.0\textwidth]{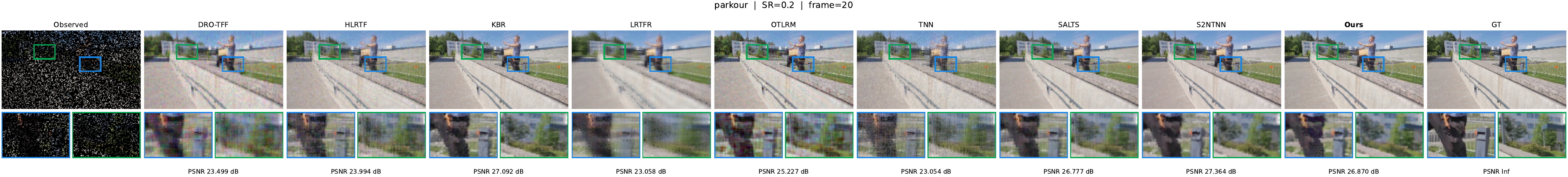}\\[1.0mm]\includegraphics[page=2,width=1.0\textwidth]{figures/supplementary/full_visuals/per_video_supplement_11col_parkour.pdf}\\[1.0mm]\includegraphics[page=3,width=1.0\textwidth]{figures/supplementary/full_visuals/per_video_supplement_11col_parkour.pdf}
\caption{RGB reconstruction comparison on Parkour at sampling rates 0.2, 0.3, and 0.4 (from top to bottom), with common ROIs across methods. The annotated PSNR values report missing-region PSNR over all 50 frames.}
\label{fig:supp-rgb-parkour}
\end{figure*}

\begin{figure*}[!t]
\centering
\includegraphics[page=1,width=1.0\textwidth]{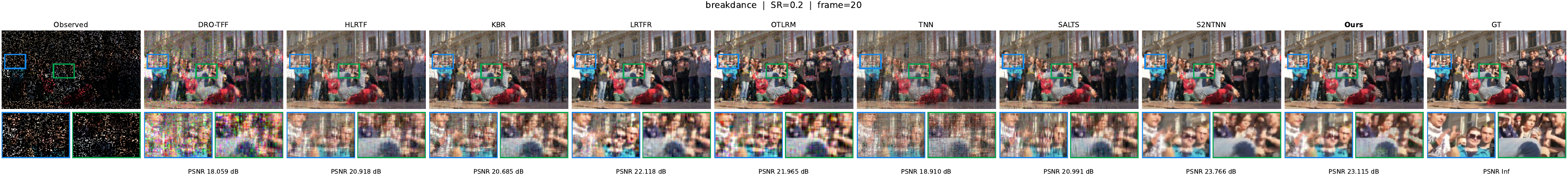}\\[1.0mm]\includegraphics[page=2,width=1.0\textwidth]{figures/supplementary/full_visuals/per_video_supplement_11col_breakdance.pdf}\\[1.0mm]\includegraphics[page=3,width=1.0\textwidth]{figures/supplementary/full_visuals/per_video_supplement_11col_breakdance.pdf}
\caption{RGB reconstruction comparison on Breakdance at sampling rates 0.2, 0.3, and 0.4 (from top to bottom), with common ROIs across methods. The annotated PSNR values report missing-region PSNR over all 50 frames.}
\label{fig:supp-rgb-breakdance}
\end{figure*}